\documentclass{article} % For LaTeX2e
\usepackage{iclr2026_conference,times}

\usepackage{hyperref}
\usepackage{url}

\usepackage{titletoc}
\usepackage{algpseudocode}
\usepackage{booktabs}
\usepackage[utf8]{inputenc} % allow utf-8 input
\usepackage[T1]{fontenc}    % use 8-bit T1 fonts
\usepackage{hyperref}       % hyperlinks
\usepackage{url}            % simple URL typesetting
\usepackage{booktabs}       % professional-quality tables
\usepackage{amsfonts}       % blackboard math symbols
\usepackage{nicefrac}       % compact symbols for 1/2, etc.
\usepackage{microtype}      % microtypography
\usepackage{xcolor}         % colors
\usepackage{wrapfig}

\usepackage[linesnumbered,ruled,vlined]{algorithm2e}
\usepackage{pifont}
\usepackage{math}

\SetKw{KwDownTo}{down to}   % algorithm2e keyword (add once in preamble)

\definecolor{papercolor}{HTML}{0668E1}
\definecolor{darkblue}{rgb}{0, 0, 0.5}
\hypersetup{colorlinks=true, citecolor=darkblue, linkcolor=darkblue, urlcolor=darkblue}

\definecolor{codepurple}{rgb}{0.58,0,0.82}

\crefname{figure}{Fig.}{Figs.}
\Crefname{figure}{Fig.}{Figs.}

\crefname{definition}{def.}{defs.}
\Crefname{definition}{Def.}{Defs.}

\crefname{algorithm}{Alg.}{Algs.}
\Crefname{algorithm}{Alg.}{Algs.}

\crefname{subsection}{Sec.}{Secs.}
\Crefname{subsection}{Sec.}{Secs.}

\crefname{section}{Sec.}{Secs.}
\Crefname{section}{Sec.}{Secs.}

\hypersetup{
    colorlinks,
    linkcolor={papercolor!75!black},
    citecolor={papercolor!75!black},
    urlcolor={papercolor!75!black}
}
\definecolor{darkred}{rgb}{0.68,0.05,0.0}
\newcommand{\samecolorfootnote}[1]{\textsuperscript{\textcolor{darkred}{#1}}}

\title{Sequential-Parallel Duality\\ in Prefix-Scannable Models}
\author{Morris Yau\samecolorfootnote{*}  \\
MIT CSAIL\\
\texttt{morrisy@mit.edu} \\
\And
Sharut Gupta\samecolorfootnote{*} \\
MIT CSAIL \\
\texttt{sharut@mit.edu} \\
\And
Valerie Engelmayer \\
TU Munich\\
\texttt{valerie.engelmayer@tum.de}\\
\And
Kazuki Irie \\
Harvard University\\
\texttt{kirie@g.harvard.edu}\\
\And
Jacob Andreas \\
MIT CSAIL\\
\texttt{jda@mit.edu}
\And
Stefanie Jegelka \\
TU Munich, MIT CSAIL \\
\texttt{stefje@mit.edu}
}
\iclrfinalcopy
\begin{document}

\maketitle

\begin{abstract}
Modern neural sequence models are designed to meet the dual mandate of parallelizable training and fast sequential inference.
Recent developments have focused on various models, such as Gated Linear Attention (GLA) and Mamba, that
achieve such ``sequential-parallel duality.''
This raises a natural question: can we characterize the full class of neural sequence models that support near-constant-time parallel evaluation and linear-time, constant-space sequential inference?
We begin by describing a broad class of such models, state space models, as those whose state updates can be computed using the classic parallel prefix scan algorithm with a custom associative aggregation operator. We then define a more general class, Prefix-Scannable Models (PSMs), by relaxing the state aggregation operator to allow arbitrary (potentially non-associative) functions such as softmax attention. This generalization unifies many existing architectures, including element-wise RNNs (e.g., Mamba) and linear transformers (e.g., GLA, Mamba2, mLSTM), while also introducing new models with softmax-like operators that achieve $O(1)$ amortized compute per token and 
$\log(N)$ memory for sequence length 
$N$.  We empirically evaluate such models on illustrative language modeling and canonical synthetic tasks, including state tracking and associative recall.
Empirically, we find that PSMs retain the functional effectiveness of transformer-based architectures while matching the inference efficiency of state space models and in some cases exhibiting better length generalization than either.\looseness=-1
% \looseness=-1\footnote{Here we'll provide a link to a public GitHub repository to release our code.}   
\end{abstract}
\footnotetext{Equal contribution.}

\section{Introduction}
%kazuki: constant "depth" not defined yet?
% \sg{Very long sentence.}
%kazuki: broke the sentence into two with a semi-column

Transformers have revolutionized sequence processing by enabling
parallelizable training over the sequence dimension \citep{trafo}---unlike classic recurrent neural networks (RNNs) \citep{elman1990finding,jordan1986,hochreiter1997long}, which require sequential training; and by handling arbitrary-length sequential dependencies with a constant parameter count---unlike convolutional neural networks, which, while parallelizable over sequence elements, require more parameters to capture longer-range dependencies \citep{GehringAGYD17,oord2016wavenet,kalchbrenner2016neural,dauphin2017lm}.
However, transformers suffer from two fundamental limitations: first, their computational and memory complexities scale quadratically with sequence length \citep{trafo, katharopoulos2020transformers}, which is particularly problematic during inference; second, they have limited expressivity, i.e., there are computations they struggle to learn to perform \citep{hahn2020theoretical,BhattamishraAG20,BhattamishraPG20,merrill2023parallelism,irie2023practical,MerrillPS24,grazzi2024unlocking,strobl2024formal,siems2025deltaproduct,movahedi2025fixed}.\looseness=-1

A body of research in neural sequence modeling has focused on developing architectures that address the primary shortcomings of transformers.
In particular, recent years have seen the introduction of diverse models that target the inference time complexity problem. In these models, the inference compute requirement is linear in time and constant in memory, just like in classic RNNs, while retaining transformer-like parallelizability during training.
Such models include \textit{element-wise recurrent models}, which are derived by simplifying either fully recurrent neural networks \citep{hochreiter1997long} (e.g.\ Quasi RNNs \citep{Bradbury17} or SRU \citep{LeiZWDA18}; see also \citep{QinYZ23,indRNN, BalduzziG16, Mozer89} or linear time-invariant dynamical systems (e.g.\ Mamba \citep{gu2023mamba})---at the cost of sacrificing expressivity \citep{MerrillWGSSY20,grazzi2024unlocking}.
Another model family has been derived from \textit{linear transformers} \citep{katharopoulos2020transformers} and \textit{fast weight programmers} \citep{schmidhuber1992learning,irie2021going}, including Gated Random Feature Attention \citep{peng2021random}, DeltaNet \citep{schlag2021linear,yang2024parallelizing}, RetNet \citep{sun2023retentive}, 
GLA \citep{YangWSPK24}, mLSTM in xLSTM \citep{beck2024xlstm}, Mamba2 \citep{DaoG24},
, and versions of RWKV \citep{peng2025rwkv7gooseexpressivedynamic}.\looseness=-1
%kazuki: RWKV is too broad. they are numbered with the same RWKV name but different versions are based on different mechanisms/equations.
Finally, concurrent work by \cite{guo2025loglinearattention} develop "log linear attention": a linear attention-style 
mechanism with $log(N)$ memory that can be applied to linear attention style mechanisms with structured memory and an efficient chunkwise-parallel primitive.

These models share the fundamental property of \textit{sequential-parallel duality} (SPD)---training is parallelizable over sequence elements, while inference is sequential and its inference time complexity is linear.
This raises a natural question: 
\textit{What is the class of neural sequence models that can be evaluated in parallel in nearly constant depth, and sequentially in nearly constant space?}

In this work we aim to characterize the family of models exhibiting SPD.  In particular, we show that these models are computable using the classic parallel prefix scan algorithm \citep{blelloch1990prefix,MartinC18} with a choice of associative aggregation operator that is specific to each model.
We define a broader model class, which we call Prefix-Scannable Models (PSMs), by generalizing the aggregation operator used in prefix scan computation. By construction, this family subsumes all existing SPD-compatible models with associative state updates. More generally, it enables the design of novel models with non-associative aggregation rules, whose per-token inference cost remains amortized $O(1)$ with memory scaling 
$O(\log(N))$ in sequence length 
$N$.
An alternate view is that PSMs are a strict generalization of RNNs: they move beyond affine state updates to support general token mixing operations---including Transformer-style self-attention over local chunks---giving rise to a novel model belonging to the PSM family, which we call Transformer-PSM.

We probe Transformer-PSM in our experiments using small but illustrative tasks: next-token prediction on WikiText-103 \citep{merity2016pointer} and synthetic algorithmic tasks that test precise state tracking and retrieval \citep{MerrillPS24,grazzi2024unlocking,li2025language,arora2024simple}.  We find that Transformer-PSMs inherit certain advantages of both Transformers and State Space Models.  They preserve the associative recall capability of Transformers, whilst exhibiting an impressive ability to track state. 
Furthermore, by varying the ``chunk'' size by which we break up a sequence of tokens, we can alter the asymptotics of a PSM from SSM-like to Transformer-like---a notion we make precise in our discussion on Sequential-Parallel Duality, which we empirically demonstrate on WikiText-103.  In summary, \looseness=-1
\begin{enumerate}
\item We define the SPD family of sequence models and unify modern linear RNNs as those with state computable by the  prefix scan algorithm with a custom choice of associative aggregator.\looseness=-1  
\item We derive a strict generalization thereof, the Prefix Scannable Models (PSMs), that admit general state aggregation functions, such as softmax attention, whilst preserving parallel training in $O(N)$ compute and $O(\log N)$ memory bound at inference.  
\item We instantiate Transformer-PSM and evaluate its abilities for state tracking, associative recall, and language modeling, using canonical sequence modeling benchmarks.
% and model language finding it to outperform canonical benchmarks.      
\end{enumerate}

% \section{Modern RNNs are Parallel Prefix Scannable}
\section{Sequence Models and Sequential–Parallel Duality}\label{sec:seq_models}

% As examples to illustrate sequential-parallel duality, we first focus on two concrete classes of sequence models that encompass many existing models satisfying the ``sequential-parallel duality'' requirement, and highlight a common property shared by all of them: they are computable using the classic parallel prefix scan algorithm \citep{blelloch1990prefix,MartinC18}.
%
% As illustrative examples of sequential–parallel duality, we consider two broad classes of sequence models, unified by a common property: both can be computed via the classic parallel prefix scan algorithm \citep{blelloch1990prefix,MartinC18}.

% 
Here we formally define sequence models and sequential–parallel duality, and provide  examples. For more details on conventions, we refer to~\Cref{sec:preliminaries}.
% %kazuki: selfTODO, need to bring back a few more things from the appendix to make this part accessible/self-contained.
% % \subsection{Preliminaries}
Throughout, let $\mathcal A$ be a finite alphabet of tokens and
$\va_{0:n-1}\in\mathcal A^{n}$ an input sequence of length $n$.  Let $\mathcal{M}$ be a latent space containing the state of a sequence model.  For example, for an RNN, this is the space of the hidden state vector.   
%
%kazuki: this'd be a 2nd order detail but I might even go one step more didactic by also defining M 
% ----------------------------------------------------------------------
First, for the sequential view, we define causal sequence models by introducing \textit{state dynamics} and \textit{inference}.\looseness=-1
% \paragraph{State dynamics and inference.}
% ----------------------------------------------------------------------
% \vspace{-0.75\baselineskip}
\begin{definition}[State kernel]\label{def:state-kernel}
A \emph{state kernel} is a map
\(
      U \colon \mathcal M\times\mathcal A \to \mathcal M
\)
with an identity element $\ve\in\mathcal M$.
It induces a \emph{state sequence}
\(
      \vs_{-1}=\ve,\;
      \vs_t = U(\vs_{t-1},\,a_t)\
\)
for $t\ge 0$.  We denote by $m(n)$ the memory required to store $\vs_{n-1}$.
\end{definition}
\begin{definition}[Inference module]\label{def:inference-module}
An \emph{inference module} is a map
\(
      F \colon \mathcal M\times\mathcal A \to \mathbb R^{|\mathcal A|}
\)
producing a distribution
\(
      \hat{\vy}_t = F(\vs_{t-1},a_t)
\)
over the next token.
\end{definition}
% \sj{should $m(n)$ and $g(n)$be defined separately (can be inline)?}
%
\begin{definition}[Sequence model]\label{def:seq-model}
A pair $(U,F)$ comprising a state kernel and an inference module is called a
\emph{causal sequence model} (or simply, \emph{sequence model}). The model's \emph{memory bound} $m(n)$ is required to evaluate $U$ and $F$ once the state
$\vs_{n-1}$ is available.
\end{definition}
Second, to formalize parallel training, we define a \textit{parallel training circuit} for sequence models.
% \paragraph{Parallel training circuit.}
% ----------------------------------------------------------------------
\begin{definition}[Parallel circuit family]\label{def:parallel-circuit}
A \emph{parallel circuit family} for a sequence model $(U,F)$ is a uniform family of 
circuits $\bigl\{C_n\bigr\}_{n\ge 1}$ such that, for all
$\va\in\mathcal A^{n}$ and all $t<n$,
$\bigl[C_n(\va)\bigr]_t =
   F\bigl(\vs_{t-1},\,a_t\bigr)$,
% \[
% \boxed{\;
%   \bigl[C_n(\va)\bigr]_t\;=\;
%   F\bigl(\vs_{t-1},\,a_t\bigr)
% \;,}
% \]
% \[\; \bigl[C_n(\va)\bigr]_t\;=\;
%   F\bigl(\vs_{t-1},\,a_t\bigr)
% \;,
% \]
where $\vs_{t-1}$ is the state (\Cref{def:state-kernel}). The model's \emph{compute bound} $T(n)$ is the size of the circuit $C_n$.\looseness=-1
\end{definition}

The circuit corresponds to the \emph{training graph}: every token can be
processed simultaneously provided sufficient parallel hardware. Together, the sequential and parallel views and their tradeoffs will characterize the Sequential–Parallel Duality (\Cref{def:spd-triplet}).

\label{sec:spd}
\begin{definition}[Sequential–Parallel Duality
     $\textsf{SPD}\bigl(T(n),\,m(n)\bigr)$]
\label{def:spd-triplet}
A sequence model $(U,F)$ is said to satisfy $\textsf{SPD}\bigl(T(n),\,m(n)\bigr)$
% \[
%   \textsf{SPD}\bigl(T(n),\,m(n)\bigr)
% \]
if the following two conditions hold:
% \jda{should it be $T_C$ above?}:
%
\begin{enumerate}
\item \textbf{Parallel training.}
      There exists a uniform circuit family
      $\{C_n\}_{n\ge 1}$ of \emph{depth} $\tilde O(1)$
      and \emph{size} $T(n)$ that realises all token‑wise predictions (\Cref{def:parallel-circuit}).
\item \textbf{Sequential inference.}
      Given $\vs_{t-1}$, the pair
      $(\vs_t,\hat \vy_t)=\bigl(U(\vs_{t-1},\va_t),F(\vs_{t-1},\va_t)\bigr)$
      is computable by a depth‑$\tilde O(1)$ circuit
      using at most $m(n)$ working memory.\looseness=-1
\end{enumerate}
\end{definition}
As {illustrative examples}, we discuss the following sequence models in light of SPD.
% \paragraph{Illustrative examples.}
% ----------------------------------------------------------------------

%\begin{remark}[
\textbf{Vanilla Transformer}: \textsf{SPD‑}$(n^{2},\,n)$.
%\label{rmk:transformer}
Training computes all $n^{2}$ attention scores in parallel with circuit depth $O(1)$ and work $T(n)=\Theta(n^{2})$. At inference, each new token requires attending to and storing all $n$ past keys/values, yielding $m(n)=\Theta(n)$ memory.

% The training graph evaluates all $n^{2}$ attention weights in \emph{parallel}
% matrix form, giving circuit depth $\Theta(\log n)$ and work
% $T(n)=\Theta(n^{2})$.  At inference a new token must attend to---and be
% stored alongside---all $n$ previous keys and values, so the per‑token
% sequential cost is $g(n)=\Theta(n)$ time and $m(n)=\Theta(n)$ memory.  Hence
% the vanilla Transformer realises \textsf{SPD‑}$(n^{2},n)$.
%\end{remark}

%\begin{remark}[
\textbf{Fully recurrent RNN}: \emph{no SPD}. %]
%\label{rmk:rnn}
A strict RNN (e.g.\ LSTM, GRU) updates its hidden state through a chain of
length~$n$.  Because each step depends on the previous one, there is \emph{no}
sub‑linear‑depth circuit that simultaneously computes every output.  Such
networks therefore fall outside the SPD framework.
%\end{remark}
%kazuki: I put this transition for now

As a preview of our results: we will additionally derive the following characterization.

%\begin{remark}[
\textbf{Prefix–Scannable and Related Models}: \textsf{SPD}-$(n,1)$ and \textsf{SPD}-$(n,\log(n))$.
%\label{rmk:psm-examples}
Modern RNN architectures that admit a Blelloch‑style scan (discussed in~\Cref{sec:psm}) for their state update
% (the general Prefix–Scannable Models of~\Cref{sec:psm})
have \emph{compute bound}
$T(n)=\Theta(n)$,
% \footnote{%
  % The scan applies to {\it chunk} representations; the constant hidden suppressed in~$\Theta$ depends on chunk size.}
parallel depth $\Theta(\log n)$, and \emph{memory bound} $m(n)=\Theta(\log n)$ or
$m(n)=\Theta(1)$, depending on whether the state size grows
logarithmically or remains constant.
%kazuki: removed discussion about chunks here; too much for a preview IMO (and prone to confusion)
% Here the scan applies to {\it chunk} representations; the constant hidden suppressed in~$\Theta$ depends on chunk size.
We therefore write
%\[
  $\textsf{SPD‑}(n,\log n) \text{ or } %\quad
  \textsf{SPD‑}(n,1)$,
%\]
both of which strictly improve on the Transformer’s linear memory latency while
retaining fully parallelisable training.
%\end{remark}

% ----------------------------------------------------------------------
\section{Prefix–Scannable Models}
\label{sec:psm}
% ----------------------------------------------------------------------
Next, we define a broad family of models that obtain a sequential-parallel duality of SPD-$(n,\log(n))$. This family consists of sequence models whose \emph{training graph} can be expressed by a \emph{Blelloch prefix scan} (see the caption in~\Cref{fig:blelloch}) over chunk
representations, followed by an independent chunk‑local prediction head. The Blelloch scan takes a sequence of tokens or chunks and an aggregation operator, and computes prefixes where the aggregator is applied over the first $n$ tokens;
it computes all prefixes in $\Theta(n)$ work and
$\Theta(\log n)$ parallel depth. We refer to~\Cref{alg:static-blelloch} in \Cref{app:algo} for the full
upsweep/downsweep algorithms.
% \emph{training graph} corresponds to a \emph{Blelloch prefix scan} (Alg.~1) over chunk
% representations, followed by an independent chunk‑local prediction head. The Blelloch scan takes a sequence of tokens or chunks and an aggregation operator, and computes prefixes where the aggregator is applied over the first $k$ tokens/chunks.
%
%The central abstraction of this paper is a family of sequence models whose
%\emph{training graph} is nothing but a Blelloch prefix scan over chunk
%representations, followed by an independent chunk‑local prediction head.  
We
call these \emph{Prefix–Scannable Models} (PSMs).  To understand the topic further we first give a brief overview of the classic parallel prefix scan.  

\begin{figure}[!htb]
    \centering
    \includegraphics[width=.85\linewidth]{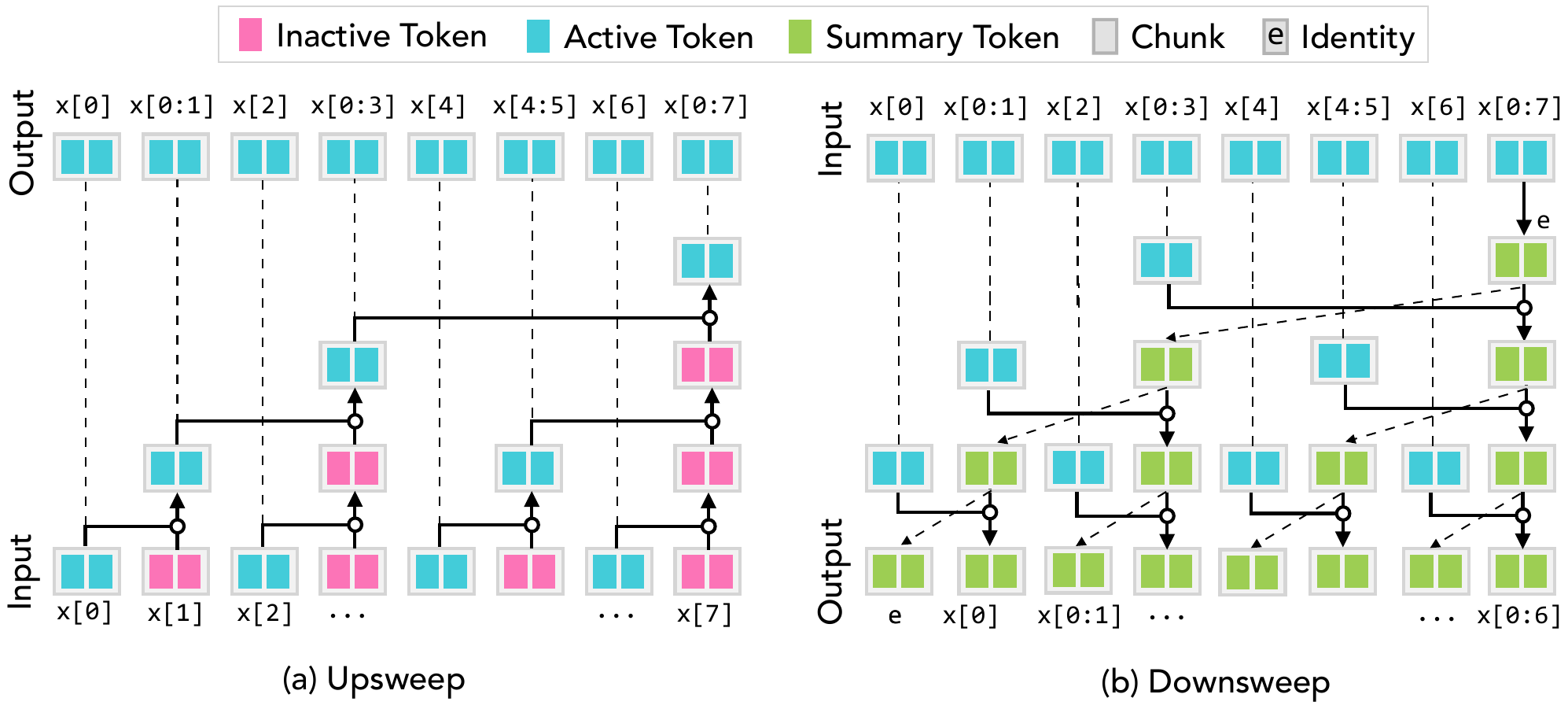}
    % \vspace{-1mm}
    \caption{
    % \kazuki{I'm assuming that all the 's' in the figures will be replaced by 'x'} \jda{consider removing the duplicated "active token" blocks in the second row of (b)}\sg{Fixed both!}
    An illustration of the Blelloch parallel scan used to compute prefix states in Prefix-Scannable Models (PSMs). Here the input has 16 tokens grouped into 8 chunks $\{\vx[0],\dots,\vx[7]\}$ (see \textbf{(a)} bottom), and the goal is to produce prefix states $\{\ve,\vx[0],\vx[0\mathpunct{:}\!1],\dots,\vx[0\mathpunct{:}\!6]\}$, where $\vx[i\mathpunct{:}\!j]$ aggregates all tokens from chunks $i$ to $j$, and $\ve$ is the identity.  
    \textbf{(a)} In the \emph{upsweep}, chunks are aggregated along a binary tree through a series of chunk aggregation operations (solid arrows), producing intermediate values and some of the final prefix states (e.g., $\vx[0\mathpunct{:}\!1],\vx[0\mathpunct{:}\!3]$).  
    \textbf{(b)} In the \emph{downsweep}, the missing prefix states are filled in by propagating values backward: $\vx[0\mathpunct{:}\!7]$ is reset to $\ve$, and copy (dotted arrows) and aggregation (solid arrows) operations complete the sequence. When each chunk is treated as an atomic element, this recovers the classic Blelloch scan.}

    \label{fig:blelloch}
    % \vspace{-8mm}
\end{figure}

\textbf{Blelloch Scan.}
Let $\mathcal M$ be a set with a binary operator $\mathsf{Agg}:\mathcal M\times\mathcal M\!\to\!\mathcal M$ and identity $\ve\in\mathcal M$.
Given $a_0,\dots,a_{n-1}\in\mathcal M$, the \emph{exclusive prefix} at index $t$ is
$P_t \coloneqq a_0 \,\mathsf{Agg}\, a_1 \,\mathsf{Agg}\,\cdots\,\mathsf{Agg}\, a_{t-1}$ (with $P_0=\ve$).
The Blelloch prefix–scan computes all $P_t$ in $O(\log n)$ parallel steps via a perfect binary tree:
(i) an \emph{upsweep} reduces adjacent pairs bottom-up until the root aggregates the whole sequence;
(ii) a \emph{downsweep} propagates prefixes top-down, using stored intermediate values so that every leaf receives its $P_t$. When the binary operator $\mathsf{Agg}$ is associative, the final prefix
array is identical to what a left-to-right sequential loop would compute.  If $\mathsf{Agg}$ is \emph{not} associative, the result is still well defined—the tree fixes a unique
\emph{parenthesisation} (see discussion in~\Cref{sec:prefix-scan})—but it may differ from the purely left-nested
order used by sequential recurrence.
We refer to this upsweep–downsweep as the \emph{static} (training) algorithm (\Cref{alg:static-blelloch}), and to the left-to-right procedure as the \emph{online} (inference) algorithm (\Cref{alg:online}), which reproduces the same tree parenthesisation.  In the next section we define Prefix Scannable Models (PSMs) by instantiating the static scan with a general choice of $\mathsf{Agg}$.

\subsection{Model Description}
% \subsection{Definition}

% \begin{algorithm}[H]
% \caption{\textsc{BlellochScan}}
% \KwIn{sequence $(\va_0,\dots,\va_{t})$, binary operator $\mathsf{Agg}^{\mathrm{Blelloch}}$,
%       identity $\ve$}
% \KwOut{exclusive prefixes $(\vs_0,\dots,\vs_{t})$ where
%         $\vs_k=\mathsf{Agg}^{\mathrm{Blelloch}}(\va_0{:}\va_{k-1})$}
% Computes all prefixes in $\Theta(t)$ work and
% $\Theta(\log t)$ parallel depth.  See
% Alg.~\ref{alg:static-blelloch} in Sec.~\ref{sec:prefix-scan} for the full
% upsweep/downsweep implementation. \jda{don't bother with an algorithm block if we're not going to include the code. but it would be nice to talk through Fig 1 here!}
% \end{algorithm}

\begin{definition}[Prefix–Scannable Model]\label{def:psm}
Fix a chunk length $c \leq n$ and partition a sequence
$\va_{0:n-1}$ into $r=n/c$ disjoint chunks
\(\mC_i=(\va_{ic},\dots,\va_{(i+1)c-1})\).
A \textbf{Prefix–Scannable Model (PSM)} is specified by three learnable
modules with depth $O(1)$:
\[
\mathsf{Enc}:\mathcal A^{c}\to\mathcal M,
\quad
\mathsf{Agg}_{\theta}:\mathcal M\times\mathcal M\to\mathcal M,
\quad
\mathsf{Inf}_{\phi}:\mathcal M\times\mathcal A^{c}\to\mathcal A^{c},
\]
and an identity element $\ve\in\mathcal M$.
%kazuki: strictly speaking e is not in M but in the `chunk' space.
\begin{enumerate}\setlength{\itemsep}{0pt}
\item \textbf{Chunk encoding}\;
      \(\vx_i=\mathsf{Enc}(\mC_i)\) for $i=0,\dots,r-1$.
\item \textbf{Prefix state}\;
      \(\{\vs_i\}_{i \in [r]}=\mathrm{BlellochScan}\bigl(\{\vx_i\}_{i \in [r]},
      \mathsf{Agg}_{\theta},\ve\bigr)\).
% \item \textbf{Chunk prediction}\;
%       \(\hat \vy_{ic{:}(i+1)c-1}=\mathsf{Inf}_{\phi}(\vs_{i-1},\mC_i)\).
\item \textbf{Chunk prediction}\;
      \(\hat \vy_{ic{:}(i+1)c-1}=\mathsf{Inf}_{\phi}(\vs_{i-1},\mC_i)\).
\end{enumerate}
\end{definition}

Note that, in terms of notation, we have $\vs_i = \vx[0\mathpunct{:}\!i]$ defined in~\Cref{fig:blelloch}.
We discuss the asymptotics of the PSM model depending on both $n$ and $c$ in~\Cref{sec:psm_details}. For now, we derive the following \emph{immediate complexity corollary} with asymptotics depending on the leading order term $n$, and focus on discussing its connections to recently proposed efficient sequence models.  ~\Cref{prop:psm-spd} follows from properties of the parallel and streaming versions of the Blelloch scan. 
%
% \subsection{Immediate complexity corollary}
%
\begin{proposition}\label{prop:psm-spd}
Every Prefix–Scannable Model is in the class 
$
  {\normalfont \textsf{SPD‑}}(n,\,\log n).
$
% \[
%   {\normalfont \textsf{SPD‑}}(n,\,\log n).
% \]
That is, its training work is $\Theta(n)$ with parallel depth
$\tilde{O}(1)$, while online inference runs in $O(1)$ amortised time and
$O(\log n)$ memory per token.
\end{proposition}

\begin{proof}[Proof Sketch]
The static Blelloch scan over $n$ chunk encodings costs linear work and $\Theta(\log n)$ depth (\Cref{alg:static-blelloch}).  The streaming evaluation
replaces that scan by the online algorithm of~\Cref{alg:online}, whose~\Cref{thm:equivalence} and~\Cref{cor:memory} show $O(1)$ amortised work and $O(\log n)$ state.
The chunk‑local $\mathsf{Inf}_{\phi}$ adds constant overhead.
\end{proof}

%kazuki: checked up to here.
%kazuki: we need to bring back the definition of E_t (but not the theorem) to here so that Table 1 makes sense; Appendix text itself needs to be adjusted to remove redundancy and flow; there are still weird self-pointers referring to appendix from appendix
% \paragraph{Road map.}

\subsection{Modern RNN Layers Fit One Affine Scan}
\label{sec:affine-unification}
% ================================================================

To relate PSMs to recent models, this section shows that a broad family of recent fast‑inference layers
(\autoref{tab:affine-zoo}) are \emph{all} PSM's.   Their state kernel can be expressed as specializations of a single associative affine
state‑update template. This enables \(\textsf{SPD‑}(n,1)\) complexity. 

%kazuki: Define R (and M much earlier)
\begin{definition}[Affine recurrence]\label{def:affine-rec}
Let \((\mathcal M,+,0)\) be an additive group and
\(\blacktriangleright:R\times \mathcal M\to \mathcal M\) a fixed bilinear action of a monoid
\((R,\circ,I)\) on \(\mathcal M\).
A layer is said to have an \emph{affine state update} if its hidden state obeys
\begin{equation}
  \vs_t \;=\; \mE_t \blacktriangleright \vs_{t-1} \;+\; \vf_t,
  \quad
  \vs_{-1}=0,
  \label{eq:affine-main}
\end{equation}
where \((\mE_t,\vf_t)\in R\times \mathcal M\) are (learnable) functions of the current
chunk $x_t$.  That is $\mE_t \coloneqq \mE_{\theta}(\vx_t)$ and $\vf_t \coloneqq \vf_{\theta'}(\vx_t)$ for learnable functions $\mE_\theta$ and $\vf_{\theta'}$.   
\end{definition}

% \subsection{One associative operator}
The models in~\Cref{tab:affine-zoo} all satisfy this affine state update template and all share the following associative aggregator.  For proof see \Cref{app:proofs}.  
\begin{lemma}[Affine aggregator](Associative Affine Aggregator) \label{lem:assoc}
Define for \((\mE_i,\vf_i)\in R\times \mathcal M\)
\[
   (\mE_2,\vf_2)\oplus(\mE_1,\vf_1)
  \;=\;
  \bigl(\mE_2\!\circ \mE_1,\;\vf_2 + \mE_2\blacktriangleright \vf_1\bigr),
  \quad
  \ve=(I,0).
\]
Then \((R\times \mathcal M,\oplus,\ve)\) is a monoid—$\oplus$ is associative with
identity $\ve$—and
\[
  (\mE_t,\vf_t)\oplus\cdots\oplus(\mE_0,\vf_0)
  \;=\;
  \bigl(\bar \mE_t,\;\vs_t\bigr),
\]
where \(\vs_t\) is the state given by~\Cref{eq:affine-main} and $\bar \mE_t$ is an auxiliary variable.  
\end{lemma}
\vspace{-2mm}
% \subsection{Theorem: all affine layers are \texorpdfstring{$\textsf{SPD‑}(n,1,1)$}{SPD‑(n,1,1)}}
Once written in the affine update form, their binary operator is associative, hence each layer is a Prefix–Scannable Model with \(\textsf{SPD‑}(n,1)\) complexity. For formal theorem and proof see ~\Cref{thm:affine-spd}.  
%Finally, the following theorem states that all affine layers are \texorpdfstring{$\textsf{SPD‑}(n,1,1)$}{SPD‑(n,1,1)}.
Importantly,
we can instantiate \Cref{def:psm} with
associative aggregators capturing learnable function families like linear dynamical systems and Gated Linear Attention.
Further discussion and the corresponding theorems can be found in \Cref{sec:affine-unification}.  Next, we turn to general PSM's, which enables new (non‑associative) aggregators, most notably softmax attention.

\begin{table}[t]
\centering
\caption{Representative examples of recently proposed layer types that cast into the affine state-update template
(\Cref{eq:affine-main}).  The same associative aggregator
\((\mE,\vf) \oplus (\mE',\vf') \mapsto (\mE \circ \mE', \vf + \mE\blacktriangleright  \vf' )\) is shared by all, and therefore, they are all in
\(\textsf{SPD‑}(n,1)\) by~\Cref{thm:affine-spd}.}
\label{tab:affine-zoo}
\footnotesize
% \scriptsize
\resizebox{1.\textwidth}{!}{%
\begin{tabular}{@{}lccc@{}}
\toprule
\textbf{Model family} &
\(\displaystyle \mE_t\blacktriangleright \vs_{t-1}\) &
\(\displaystyle \vf_t\) &
 Gate / operator \\ \midrule
Linear Attention \citep{katharopoulos2020transformers}                & \(\vs_{t-1}\)                       & \(\vv_t \vk_t^{\!\top}\) & identity \(I\) \\
DeltaNet \citep{schlag2021linear}  & \(\vs_{t-1}(\mI-\beta_t \vk_t \vk_t^{\!\top})\) & \(\beta_t \vv_t \vk_t^{\!\top}\) & projector \\
Gated DeltaNet \citep{yang2024gated} & \(\alpha_t  \vs_{t-1}(\mI-\beta_t \vk_t \vk_t^{\!\top})\) & \(\beta_t \vv_t \vk_t^{\!\top}\) & projector \\
% DeltaNet / Gated DeltaNet       & \(s_{t-1}(I-\beta_t k_t k_t^{\!\top})\) & \(\beta_t v_t k_t^{\!\top}\) & projector \\
RetNet \citep{sun2023retentive} & \(\gamma \vs_{t-1}\)              & \(\vv_t \vk_t^{\!\top}\) & scalar gate \(\gamma\) \\
mLSTM  \citep{beck2024xlstm} & \(f_t \vs_{t-1}\)    & \(i_t \vv_t \vk_t^{\!\top}\) & scalar gate $f_t$ \\
Gated RFA \citep{peng2021random}  & \(g_t \vs_{t-1}\)                   & \((1-g_t) \vv_t \vk_t^{\!\top}\) & scalar gate $g_t$ \\
% RetNet / RWKV / Mamba‑2         & \(\gamma_t s_{t-1}\)              & \(v_t k_t^{\!\top}\) & scalar gate \(\gamma_t\) \\
S4 / S6  \citep{GuGR22} & \(e^{-\bm{\alpha}_t}\!\odot \vs_{t-1}\)   & \(\mB\odot(\vv_t 1^{\!\top})\) & diagonal gate \\
% GLA, mLSTM   & \((1-\alpha_t)\!\odot s_{t-1}\)    & \(v_t k_t^{\!\top}\) & diagonal gate \\
Mamba  \citep{gu2023mamba}  &  \(\bar{A}(\vx_t) \vs_{t-1}\) & \(\bar{B}(\vx_t)\vx_t\) & diagonal gate \\
GLA  \citep{YangWSPK24}  & \(1\bm{\alpha}_t^{\!\top} \!\odot \vs_{t-1}\)   & \(\vv_t \vk_t^{\!\top}\) & diagonal gate \\
\bottomrule
\end{tabular}
}
% \vspace{-2mm}
\end{table}
%kazuki: see my comments above. All variations should be defined!

\subsection{Beyond Affine State Recurrence: PSMs with General Aggregation}
\label{sec:prefix-scan}

The \emph{parallel prefix–scan} computes per-position prefixes with $\mathcal O(n)$ work and $\mathcal O(\log n)$ depth when the binary operator is associative~\citep{blelloch1990prefix}. We generalize this view to \emph{non-associative} operators (e.g., softmax attention).  For a longer discussion with extensive proofs, see \Cref{sec:prefix-scan}.  

The key issue is \emph{parenthesisation}: for a non-associative $\mathsf{Agg}$, different groupings of
$x_0 \,\mathsf{Agg}\, x_1 \,\mathsf{Agg}\,\cdots\,\mathsf{Agg}\, x_{t-1}$
produce different values. The Blelloch scan resolves this by fixing a single full binary tree (upsweep/downsweep), hence a unique parenthesisation.
Let
\begin{equation}
  \mathsf{Agg}:\ \mathcal M \times \mathcal M \to \mathcal M,\qquad
  \text{identity }\ve\in\mathcal M,
  \label{eq:agg-nonassoc}
\end{equation}
with no associativity assumption unless stated. Define $\pi_{\text{Blelloch}}$ as the binary-tree parenthesisation induced by the static scan.
The \emph{static} Blelloch scan (\Cref{alg:static-blelloch}) computes, for every $t$,
\[
  \vs_t \;=\; (x_0 \,\mathsf{Agg}\, x_1 \,\mathsf{Agg}\,\cdots\,\mathsf{Agg}\, x_{t-1})
  \;\;\text{evaluated under } \pi_{\text{Blelloch}}.
\]
This matches the sequential left-to-right recurrence when $\mathsf{Agg}$ is associative; otherwise it is still well-defined value for the fixed tree. The work is $\mathcal O(n)$, and depth is $\mathcal O(\log n)$.

\textbf{Online binary counter (inference).}
The \emph{online} variant (\Cref{alg:online}) maintains at most one root per block size $2^k$; inserting $x_t$ performs the usual binary carry with $\mathsf{Agg}$. The current prefix is the most significant bit (MSB) $\to$ least significant bit (LSB) fold of occupied roots. This reproduces \emph{exactly} $\pi_{\text{Blelloch}}$ for each $t$ while using $\mathcal O(\log n)$ memory.

\label{app:algo}
\begin{minipage}{0.49\textwidth}
\begin{algorithm}[H]
% \begin{algorithm}[!htb]
\footnotesize
  \caption{ $\textsc{StaticBlellochScan}$}\label{alg:static-blelloch}
  \KwIn{$\bigl(\{\vx_i\},\mathsf{Agg}_{\theta},\ve\bigr)$: Array of encoded chunks $\vq[\vx_1\,...\,\vx_{r-1}]$ with $r=2^{k}$ (power of two) chunks;
        operator $\mathsf{Agg}$ with identity $\ve$}
  \KwOut{Exclusive prefixes written back into $\vq$}

  \medskip\textbf{Representation.}  Store the complete binary tree in the usual
  heap layout $T[1\,..\,2n-1]$:
  \begin{enumerate}
    \item leaves $T[n+i]\leftarrow \vq[i]$ for $i=0,\dots,r-1$;
    \item an internal node $v$ has children $2v$ and $2v{+}1$.
  \end{enumerate}

  \medskip\textbf{Upsweep (reduction).}
  \ForPar{$v\gets n-1$ \KwDownTo $1$}{
        $T[v] \gets \mathsf{Agg}\!\bigl(T[2v],\,T[2v+1]\bigr)$
  }

  \medskip\textbf{Downsweep (prefix propagate).}
  Allocate $P[\;]$; set $P[1]\gets \ve$ \tcp*[l]{root gets identity}
  \ForPar{$v\gets 1$ \KwTo $n-1$}{
        $P[2v]   \gets P[v]$\;
        $P[2v+1] \gets \mathsf{Agg}\!\bigl(P[v],\,T[2v]\bigr)$
  }

  \medskip\textbf{Write back.}
  \ForPar{$i\gets 0$ \KwTo $n-1$}{
        $\vq[i] \gets P[n+i]$
  }
\end{algorithm}
\end{minipage}
% \hfill
\hspace{4mm}
\begin{minipage}{0.48\textwidth}
\begin{algorithm}[H]
% \begin{algorithm}[!htb]
 \footnotesize  \caption{$\textsc{BinaryCounterUpdate}$}\label{alg:online}
  \KwIn{$(\texttt{root},\vx,\mathsf{Agg}_{\theta},\ve)$: Stream of encoded chunks $\vx_0,\dots,\vx_{r-1}$; operator $\mathsf{Agg}$ with identity $\ve$}
  \KwOut{Prefix value $\vp_t$ for each $t$ (Blelloch parenthesisation)}

  \textbf{State:}\\ \texttt{root[$k$]} stores the root of the current block of
  size $2^k$ or is \texttt{empty} initialise all to \texttt{empty}.

  \medskip
  \For{$t\gets 0$ \KwTo $r-1$}{
        $carry \gets \vx_t$\\ $k\gets 0$\\
        \While{$\texttt{root}[k]\neq\texttt{empty}$}{
              $carry \gets \mathsf{Agg}\!\bigl(\texttt{root}[k],\,carry\bigr)$\\
              \texttt{root[$k$]} $\gets \texttt{empty}$\\
              $k \gets k{+}1$\\
        }
        \texttt{root[$k$]} $\gets carry$\tcp*{place merged tree}

        $\vp \gets \ve$\\
        \For{$k\gets \lfloor\log_2(t{+}1)\rfloor$ \KwDownTo $0$}{
              \If{$\texttt{root}[k]\neq\texttt{empty}$}{
                   $\vp \gets \mathsf{Agg}(\vp, \,\texttt{root}[k])$\\}}
        \textbf{emit} $\vp$\\
  }
\end{algorithm}
\end{minipage}

Together, the static and online scans yield PSMs in $\text{SPD}(n,\log n)$: linear work for training and logarithmic memory for streaming inference. (See \Cref{fig:inference} for the chunked Transformer-PSM inference.) We obtain the following \textit{correctness and complexity analysis}. We defer proofs to \Cref{sec:prefix-scan}

\begin{figure}[t]
    \centering
    \includegraphics[width=0.9\linewidth]{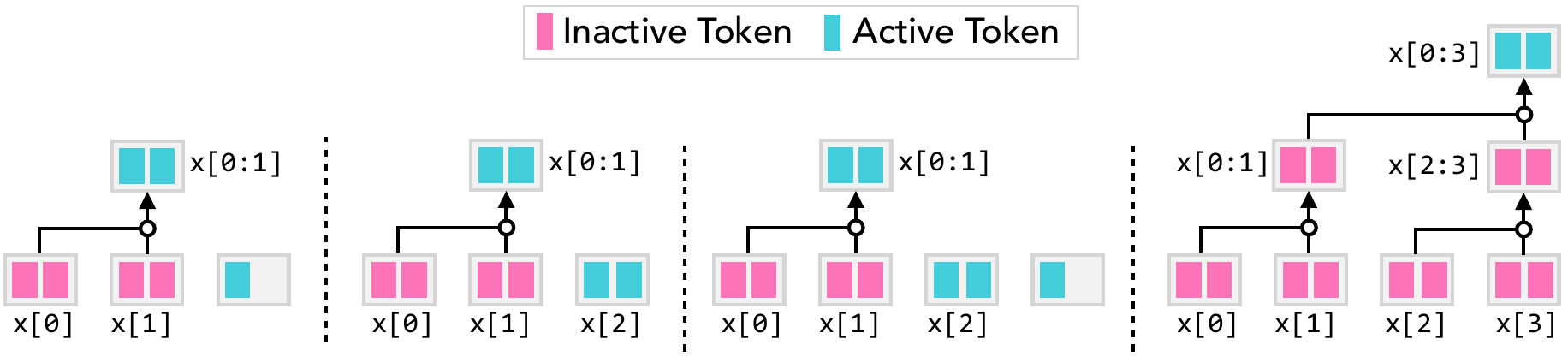}
    \vspace{-2mm}
    \caption{An illustration of the autoregressive state computation of ``Transformer-PSM'' (\Cref{sec:exp}) at inference time. Here the model uses a chunk size of 2.
    From left to right, a single new token is fed to the model at a time. \textbf{Two first figures:} when predicting tokens in chunk $\vx[2]$, the model only requires tokens from the prefix state $\vx[0\mathpunct{:}\!1]$ and those within $\vx[2]$. \textbf{Third figure:} predicting tokens in chunk $\vx[3]$ requires the prefix state $\vx[0\mathpunct{:}\!1]$, and chunks $\vx[2]$ and $\vx[3]$. \textbf{Last figure:} once all tokens in chunk $\vx[3]$ are processed, a new prefix state $\vx[0\mathpunct{:}\!3]$ is computed, which is later used to predict tokens in $\vx[4]$, and so on. Prefix state $\vs_i$ corresponds to $\vs_i = \vx[0\mathpunct{:}\!i]$.}
    % \caption{Caption \jda{consider merging this and the above into one figure with two subfigures, and use different colors to avoid overloading. what exactly are we trying to show here?}}
    \label{fig:inference}
    % \vspace{-9mm}
\end{figure}

\begin{restatable}[Equivalence with static Blelloch]{theorem}{EquivBlelloch}\label{thm:equivalence}
Let $\vp_t$ be the value emitted at time $t$ by~\textup{\Cref{alg:online}}.  Then $\vp_t$ equals the exclusive prefix
returned by the static Blelloch scan, \emph{regardless of whether
$\mathsf{Agg}$ is associative}.
\end{restatable}

\begin{restatable}[Memory bound]{corollary}{MemoryBound}\label{cor:memory}
After $t{+}1$ chunks \textup{\Cref{alg:online}} stores at most
$\lceil\log_2(t{+}1)\rceil$ root values; hence the worst–case space usage is
$\mathcal O(\log n)$.
\end{restatable}
\textbf{Work.}  Inserting a new element touches exactly the trailing
\texttt{1}--bits of $t$; the expected number of such bits is~$2$, so the
amortised number of $\mathsf{Agg}$ calls per element is constant.

Together,~\Cref{thm:equivalence} and~\Cref{cor:memory} show
that the online binary–counter scan is an \emph{optimal–space, streamable}
realisation of the Blelloch parenthesisation, {extending prefix–scan techniques
to non–associative operators without increasing asymptotic cost in time}.  This
flexibility enables a larger class of \emph{prefix–scannable models}: sequence models
whose per–token state update is any binary operator that admits
$\mathcal O(\log n)$ space $O(1)$ time online evaluation via the mechanism above. We provide further analytical details of PSMs in \Cref{sec:psm_details}.

% PSMs use a prefix-scan structure that enables polylog-depth training and logarithmic-memory online decoding, unifying scan-friendly models under a common framework (see~\Cref{sec:psm_details} for analytical details).

\subsection{Transformer-PSM}
\label{sec:tpsm}

In this section we instantiate $\mathsf{Enc}$, $\mathsf{Agg}$, $\mathsf{Inf}$ to concretely define the Transformer-PSM architecture that we use to run our empirics in \Cref{sec:exp} to validate our theoretical predictions.  The model is uniquely specified by the following three modules. 
% \paragraph{Transformer-PSM modules} A Transformer-PSM is uniquely specified by the hyperparameters for the three modules.  

\textbf{Encoder} ($\mathsf{Enc}$): This is a simple embedding layer that transforms discrete vocabulary tokens into continuous vectors, implemented as a standard nn.embedding layer.

\textbf{Aggregation} $(\mathsf{Agg}_\theta)$. A GPT-2 style Transformer (hidden dim $d$, $H$ heads, $L$ layers) with a \emph{bidirectional} attention mask,
$
\mathsf{GPT}^{b}_\theta:\ \mathbb{R}^{d\times 2c} \to \mathbb{R}^{d\times 2c}.
$
% \[
% \mathsf{GPT}^{b}_\theta:\ \mathbb{R}^{d\times 2c} \to \mathbb{R}^{d\times 2c}.
% \]
Given two chunk states $x_i,x_j\in\mathbb{R}^{d\times c}$, define token-concat $[x_i\,|\,x_j]\in\mathbb{R}^{d\times 2c}$ and the \emph{right-half slice}
$\mathrm{RH}(Y)\coloneqq Y[\,:\,,\,c:2c]\in\mathbb{R}^{d\times c}$. We write
% $
% \mathsf{Agg}_\theta(x_i,x_j)\ \coloneqq\ \mathrm{RH}\!\bigl(\mathsf{GPT}^b_\theta([x_i\,|\,x_j])\bigr)\ \in \mathbb{R}^{d\times c}.
% $
\[
\mathsf{Agg}_\theta(x_i,x_j)\ \coloneqq\ \mathrm{RH}\!\bigl(\mathsf{GPT}^b_\theta([x_i\,|\,x_j])\bigr)\ \in \mathbb{R}^{d\times c}.
\]

\textbf{Inference} $(\mathsf{Inf}_\phi)$. A GPT-2 style Transformer (hidden dim $d$, $H$ heads, $L$ layers) with a \emph{causal} mask,
$
\mathsf{GPT}^{c}_\phi:\ \mathbb{R}^{d\times 2c} \to \mathbb{R}^{d\times 2c}.
$
% \[
% \mathsf{GPT}^{c}_\phi:\ \mathbb{R}^{d\times 2c} \to \mathbb{R}^{d\times 2c}.
% \]
Given a prefix state $\vs_{t-1}\in\mathbb{R}^{d\times c}$ and a token chunk $\mathsf{Enc}(C_t) \in \mathbb{R}^{d \times c}$,
\[
 \mathsf{Inf}(\vs_{t-1},\vx_t)\ \coloneqq\ \mathrm{RH}\!\bigl(\mathsf{GPT}^{c}_\phi([\,\vs_{t-1}\,|\,\mathsf{Enc}(C_t)\,])\bigr)\ \in \mathbb{R}^{d\times c},
\]
which we interpret as per-token logits for $C_t[1:]$ (next-token prediction within the chunk).  Once the three modules are defined, we train Transformer-PSM with \Cref{alg:psm-train} and inference with \Cref{alg:psm-decode}.

%-------------- Algorithms --------------
\begin{minipage}{0.39\textwidth}
\begin{algorithm}[H]
% \begin{algorithm}[!htb]
\footnotesize
  \caption{Transformer-PSM Training (static scan over chunks)}\label{alg:psm-train}
  \KwIn{Sequence of tokens $\va_{0:n}$, \(\mathsf{Enc},\mathsf{Agg}_{\theta},
        \mathsf{Inf}_{\phi}\), chunk size $c$}
  \KwOut{Predictions $\hat{\vy}_{0:n}$}

  $r \gets n/c$\tcp*{number of chunks}

  \medskip
  \ForPar{$i \gets 0$ \KwTo $r$}{\label{line:chunk-enc}
        $\vx_i \gets \mathsf{Enc}\!\bigl(\va_{ic{:}(i+1)c-1}\bigr)$
  }

  $\{\vs_i\}_{i=0}^{r} \gets
      \textsc{StaticBlellochScan}\bigl(\{\vx_i\},\mathsf{Agg}_{\theta},\ve\bigr)$ \tcp*{~\Cref{alg:static-blelloch}}\

  % \medskip
  \ForPar{$i \gets 0$ \KwTo $r$}{\label{line:inf}
        $\hat{\vy}_{ic{:}(i+1)c-1} \gets
           \mathsf{Inf}_{\phi}\!\bigl(\vs_{i-1},\,\va_{ic{:}(i+1)c-1}\bigr)$
  }
\end{algorithm}
\end{minipage}
% \hfill
\hspace{6mm}
\begin{minipage}{0.56\textwidth}
\begin{algorithm}[H]
% \begin{algorithm}[!htb]
\footnotesize
  \caption{Transformer-PSM Inference (binary‑counter scan)}\label{alg:psm-decode}
  \KwIn{Streamed tokens $\va_t$, \(\mathsf{Enc},\mathsf{Agg}_{\theta},
        \mathsf{Inf}_{\phi}\), chunk size $c$}
  \KwOut{Streaming predictions $\hat{\vy}_t$}

  \textbf{State:}\\
  \quad\(\texttt{root}[k]\gets\texttt{empty}\) for all $k$ (cf.\ \Cref{alg:online})\\
  \quad\(\texttt{buf}\gets[\,]\)  \tcp*{collects current chunk}

  \medskip
  \For(\tcp*[f]{token index $t=0,1,\dots$}){each $\va_t$}{
        append $\va_t$ to \texttt{buf}\;

        \If(\tcp*[f]{completed one chunk}){$|\texttt{buf}|=c$}{
            $\vx \gets \mathsf{Enc}(\texttt{buf})$\;
            $\vs \gets \textsc{BinaryCounterUpdate}%
                    (\texttt{root},\vx,\mathsf{Agg}_{\theta},\ve)$\tcp*{~\Cref{alg:online}}
            $\hat{\vy}_{t-c+1:t} \gets \mathsf{Inf}_{\phi}(\vs,\texttt{buf})$\;
            clear \texttt{buf}\;
        }
  }
\end{algorithm}
\end{minipage}

\section{Experimental Results}
\label{sec:exp}
%kazuki: we need a transition here; ideally explain what's the goal of these experiments.
%kazuki: I put something, please read and correct!
% Next, we empirically evaluate the capabilities of the novel sequence model derived from the prefix scannable model class (\Cref{sec:psm_details}) by using a non-associative aggregation operator.
% We focus on a model defined by using the Transformer's causal self-attention operation as the concrete choice of both $\mathsf{Agg}_{\theta}$ and $\mathsf{Inf}_{\phi}$---we refer to this model as ``Transformer-PSM'' (T-PSM; see~\Cref{fig:inference} for an illustration of its inference). 
% More specifically, $\mathsf{Agg}_{\theta}$ is applied to a concatenation of two chunks, which first produces a sequence with the length of two chunks; unless otherwise stated, we drop the first half of the produced sequence to obtain an output with the length of a single chunk.
% The same aggregation $\mathsf{Agg}_{\theta}$ is applied at different nodes of the tree during both the upsweep and downsweep steps, resulting in parameter sharing across depth \citep{dehghani2018universal}.\looseness=-1

The main goal of our experiments is to evaluate and explore the capabilities and properties of Transformer-PSM (\Cref{sec:tpsm}).
For this, we conduct experiments on representative sequence learning tasks: a synthetic algorithmic task requiring state-tracking (\Cref{exp:s5}), a synthetic task for associative recall, and language modeling (\Cref{exp:lm}).
Each experiment was conducted on a single NVIDIA V100-32GB GPU.
All experiments were implemented using the PyTorch framework \citep{NEURIPS2019_bdbca288}.\looseness=-1

%kazuki: adjust alignment
% \begin{figure}[!htb]
\begin{figure}[t]
\centering
\begin{minipage}{0.48\textwidth}
\centering
\includegraphics[width=0.9\textwidth]{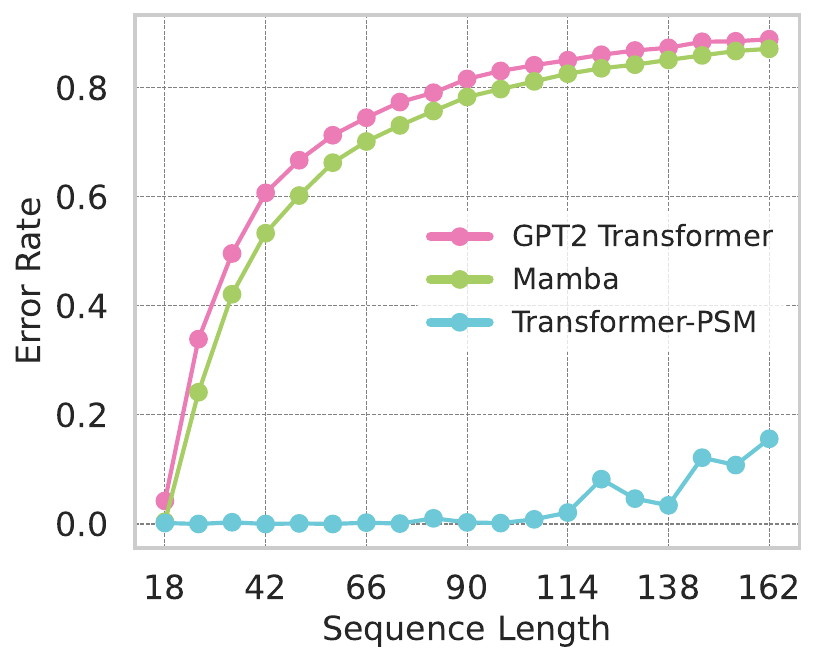}
\caption{Error rate on the state tracking $S_5$ task. After training on sequences with lengths up to 18, Transformer-PSM generalizes to more than 160 tokens, far beyond Transformer and Mamba.}
\label{fig:state-tracking}
\end{minipage}
\hfill
\begin{minipage}{0.48\textwidth}
\centering
\includegraphics[width=0.9\textwidth]{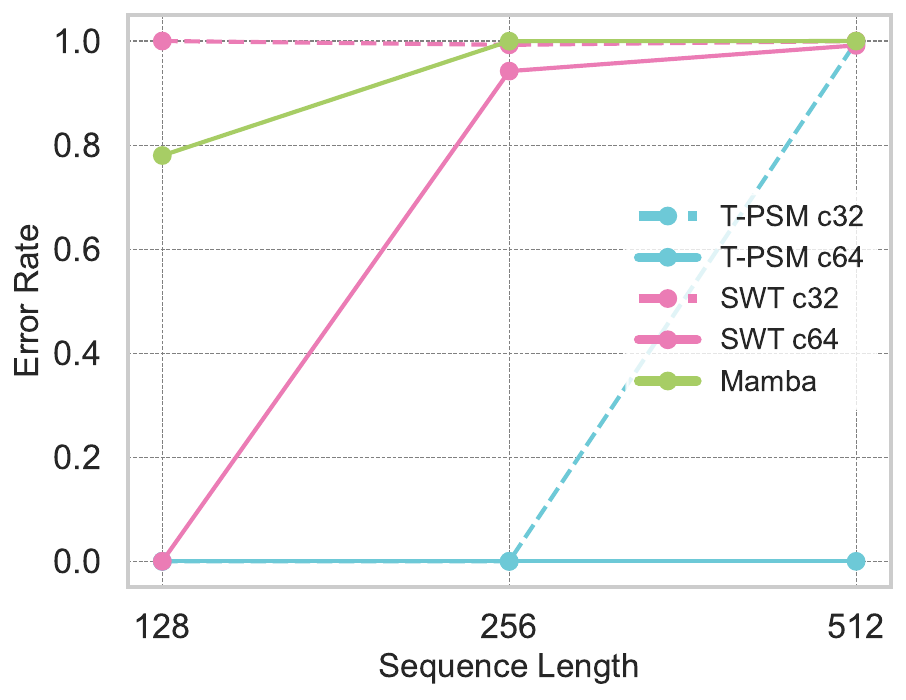}
\caption{Error rate on MQAR of Transformer-PSM (T-PSM), Sliding Window Transformer (SWT) and Mamba. Evaluated lengths are in-distribution.} %\jda{can we show this as an error rate rather than an accuracy to align with other graphs?}
\label{fig:mqar}
\end{minipage}
\vspace{-5mm}
\end{figure}

% \begin{wrapfigure}{r}{0.4\textwidth}
%     \centering
%     \vspace{-15mm}
%     \includegraphics[width=1\linewidth]{figures/err_vs_seq_len.pdf}    
%     \caption{Error rate on MQAR of Transformer-PSM (T-PSM), Sliding Window Transformer (SWT) and Mamba. Evaluated lengths are in-distribution. %\jda{can we show this as an error rate rather than an accuracy to align with other graphs?}
%     }
%     \vspace{-5mm}
%     \label{fig:mqar}
% \end{wrapfigure}

\subsection{State Tracking $S_5$}
\label{exp:s5}
The $S_5$ state tracking problem ~\citep{MerrillPS24, kim2023, li2025language} is the formal version of the ``permute cups and balls'' challenge, 
% \begin{wrapfigure}{r}{0.4\textwidth}
%     \centering
%     % \vspace{-15mm}
% \includegraphics[width=0.4\textwidth]{figures/state_tracking.pdf}
%     \caption{Error rate on the state tracking $S_5$ task. After training on sequences with lengths up to 18, Transformer-PSM generalizes to much longer sequences with more than 160 tokens, far beyond GPT2-Transformer and Mamba.
%     }
%     % \vspace{-15mm}
%     \label{fig:state-tracking}
% \end{wrapfigure}
where a sequence of permutations is composed, with the objective of tracking the resulting permutation at each time step. Problems as diverse as tracking finite state automata and evaluation of boolean expressions can be reduced to this task.
Naturally, as the sequence of permutations lengthens, this task becomes increasingly difficult for a constant-depth model that has a constant budget for sequential computation. Indeed, the $S_5$ state tracking task is $\mathsf{NC}^1$ complete \citep{Barrington1986}.  
It is known to be difficult for both standard Transformers and linear RNNs such as Mamba \citep{MerrillPS24,grazzi2024unlocking}.

%because of the marginal stability of the permutation matrix.  

%The $S_5$ task is designed to evaluate a sequence model's ability to accurately maintain the state of a system that exhibits complex dependencies over time. 
%The $S_5$ alludes to the symmetric group of degree 5, a non-abelian group with 120 elements each corresponding to a permutation. For example the token $'13254'$ encodes a permutation  that swaps the third and second elements and swaps the fifth and fourth elements.  The $S_5$ task is to predict the state at every timestep given a sequence of these permutation tokens.    

We train from scratch on sequences of length 4 to 18 in a curriculum and subsequently evaluate on lengths up to 180 to test for length generalization.  We generate 100,000 sequences per length and train for 20 epochs for each of three different models: (1) a standard GPT2 model with 12 layers, 12 heads, 768 hidden dimensions; (2) a 370M-parameter Mamba model with 48 layers and a 1024-dimensional hidden state; (3) Transformer-PSM with $(d=768,H=1,L=1)$ for $\mathsf{Agg}$, $(d=768,H=1,L=1)$ layer for $\mathsf{Inf}$, and chunk size $c=1$.  
All models are trained with Adam with dropout 0.1, weight decay 0.01, learning rate $10^{-4}$.   

\Cref{fig:state-tracking} shows the results. We find that whilst Mamba slightly outperforms GPT2, the new T-PSM has remarkably low error rate even for sequences significantly longer than those observed during training, showing that these models exhibit strong length generalization for state tracking tasks. 

% \begin{figure}[t]
%     \centering
%     \includegraphics[width=0.5\textwidth]{figures/state_tracking.pdf}
%     \caption{Error rate on the state tracking task. After training on sequences up to length 18, Transformer-PSM generalizes to much longer sequences with more than 160 tokens, far beyond GPT2-Transformer and Mamba. 
%     % \val{Maybe mention which task this is in the caption?}\kazuki{great idea. done}
%     }
%     \label{fig:state-tracking}
% \end{figure}

% \b egin{figure}[!htb]
% \centering
% \begin{minipage}{0.5\textwidth}
% \centering
% \includegraphics[width=\textwidth]{figures/state_tracking.pdf}
% \caption{}
% \label{fig:state-tracking}
% \end{minipage}
% \end{figure}

\subsection{Multi Query Associative Recall (MQAR)}
\label{exp:mqar}
% \begin{figure}[!htb]
% \centering
% \begin{minipage}{0.49\textwidth}
% \centering
% \includegraphics[width=0.8\textwidth]{figures/err_vs_seq_len.pdf}
% \caption{In-distribution error rate of Transformer-PSM (T-PSM), Sliding Window Transformer (SWT) and Mamba on MQAR.}
% \label{fig:mqar_in}
% \end{minipage}
% \hfill
% \begin{minipage}{0.49\textwidth}
% \centering
% \includegraphics[width=0.9\textwidth]{figures/len_general.pdf}
% \caption{Error rates on MQAR after training on sequences of length 256. With small enough chunk sizes, T-PSM generalizes to longer sequences than a Transformer. }
% \label{fig:mqar-out}
% \end{minipage}
% \end{figure}

% \begin{wrapfigure}{r}{0.4\textwidth}
%     \centering
%     \vspace{-15mm}
%     \includegraphics[width=1\linewidth]{figures/err_vs_seq_len.pdf}    
%     \caption{Error rate on MQAR of Transformer-PSM (T-PSM), Sliding Window Transformer (SWT) and Mamba. Evaluated lengths are in-distribution. %\jda{can we show this as an error rate rather than an accuracy to align with other graphs?}
%     }
%     \vspace{-5mm}
%     \label{fig:mqar}
% \end{wrapfigure}
%kazuki: TODO if we will have no other figures, this one should be put side by side with Figure 3; to avoid weird spacing in the next page

In Associative Recall, the task is to recall whatever value followed a key earlier in a given sequence. MQAR extends this task to multiple such key-value pairs to increase the memory demand \citep{zoology}.
While constant state size recurrent models struggle with this task, a 2-layer transformer excels by solving it perfectly.
% With its $O(n^2)$ state size, a 2-layer vanilla transformer easily solves this task perfectly, while recurrent models with their $O(1)$ memory struggle. 
To gauge where on this spectrum our model falls, we train different models on MQAR for 64 epochs with vocabulary size 8192 and 8 key-value pairs. %\kazuki{I'm not sure I understand the sentence below. maybe remove?} \val{Yeah I had mixed up the roles of queries and values. It means we changed the default parameters for sampling the synthetic data to make the task harder. It is (probably) the reason why Mamba cannot do it, even though there is literature in which it can. So I wouldn't remove, hope it is clearer now?}
In the typical setting of this task, sequences are constructed in a way that a key is queried shortly after it appears for the first time; here we do not use such a bias and sample queries uniformly, which makes the task harder than the standard setting.
% In most of the literature, keys are biased to be queried shortly after appearing for the first time. This simplifies the task especially for models with short attention spans and little memory. We remove this bias and sample gap sizes uniformly.

Here we instantiate Transformer-PSM with $(d=256,H=1,L=2)$ $\mathsf{Agg}$, $(d=256,H=1,L=2)$ $\mathsf{Inf}$.  We also use a learnable linear projection to compress the chunks instead of taking the right half. The chunk size is 32 or 64.
For comparison, we include both Mamba and Sliding Window Transformer (SWT) baselines \citep{beltagy2020longformerlongdocumenttransformer, zaheer2020big}.  The
SWT is a GPT2 model with $(d=256,H=1,L=4)$, where we use a sliding window size of 32 or 64.
% To compare to a different variant of sparse attention, we report performance of Sliding Window Transformers (SWT) \citep{beltagy2020longformerlongdocumenttransformer, zaheer2020big} with the same parameters.

% Even though T-PSM maintains a smaller state of size $O(c \cdot \log (n/c))$ compared to SWT's $O(n)$, they require a smaller chunk size $c$ to solve the task, as can be seen in \Cref{fig:mqar}.

\Cref{fig:mqar} shows the results. 
Here all the evaluation lengths are in the training distribution. We find that T-PSM with a chunk size of 64 achieves the perfect accuracy like the full context transformer, while reducing its chunk size to 32 yields performance degradation on a long length (512).
Mamba fails in our setting; unlike in prior work \citep{arora2024simple, okpekpe2025revisiting}, our setting is harder due to our uniform query sampling as discussed above.

\subsection{Language Modeling on WikiText-103 with Transformer-PSM}
\label{exp:lm}
% Following prior works~\citep{gu2023mamba,yang2024gated}, 
Here we evaluate perplexity on the WikiText-103 dataset \citep{merity2016pointer}. We benchmark Transformer-PSM $(d=768,H=12,L=1)$ $\mathsf{Agg}$, $(d=768,H=12,L=11)$ $\mathsf{Inf}$, by varying the self-attention chunk size from 32 to 256 tokens and measuring test perplexity against the vanilla GPT-2 (base) baseline with a context size of 512 trained from scratch. As shown in~\Cref{fig:wikitexT_Chunk}, as the chunk size grows, perplexity falls gracefully from 24.12 at 32 tokens to 22.45 at 256 tokens---closely approaching vanilla GPT-2’s perplexity of 22.28---demonstrating that larger chunks recover nearly full-context modeling power while preserving our model’s linear-time inference. 
For reference, we also include a baseline for Mamba trained from scratch at $130m$ parameters, $768$ hidden dimension, trained for $10$ epochs with the same optimizer hyperparameters achieving a ppl of $24.7$.   

\begin{figure}[t]
\centering
\begin{minipage}{0.48\textwidth}
\centering
\includegraphics[width=0.9\textwidth]{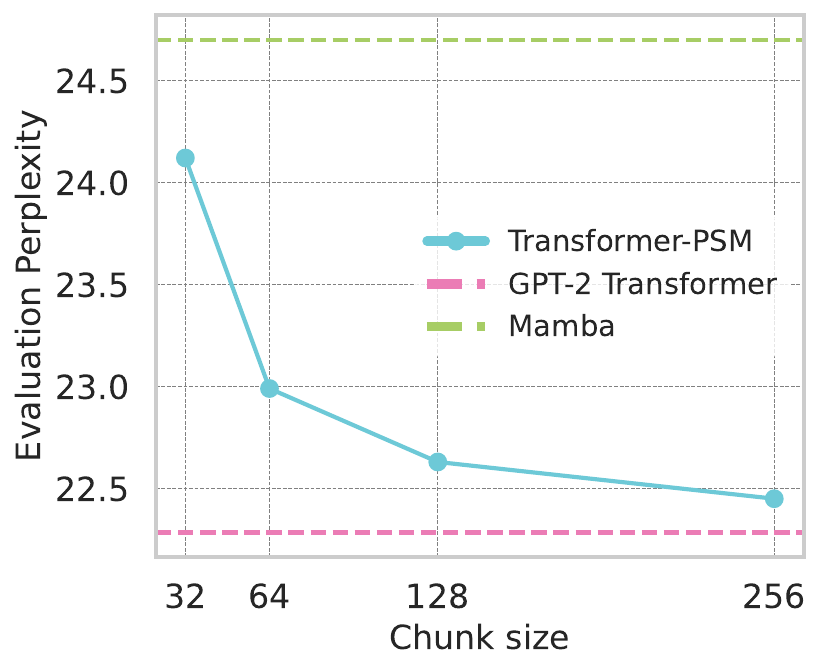}
\vspace{-3mm}
\caption{Evaluation perplexity of Transformer-PSM with varying chunk sizes on WikiText-103}
\label{fig:wikitexT_Chunk}
\end{minipage}
\hfill
\begin{minipage}{0.48\textwidth}
\centering
\includegraphics[width=0.9\textwidth]{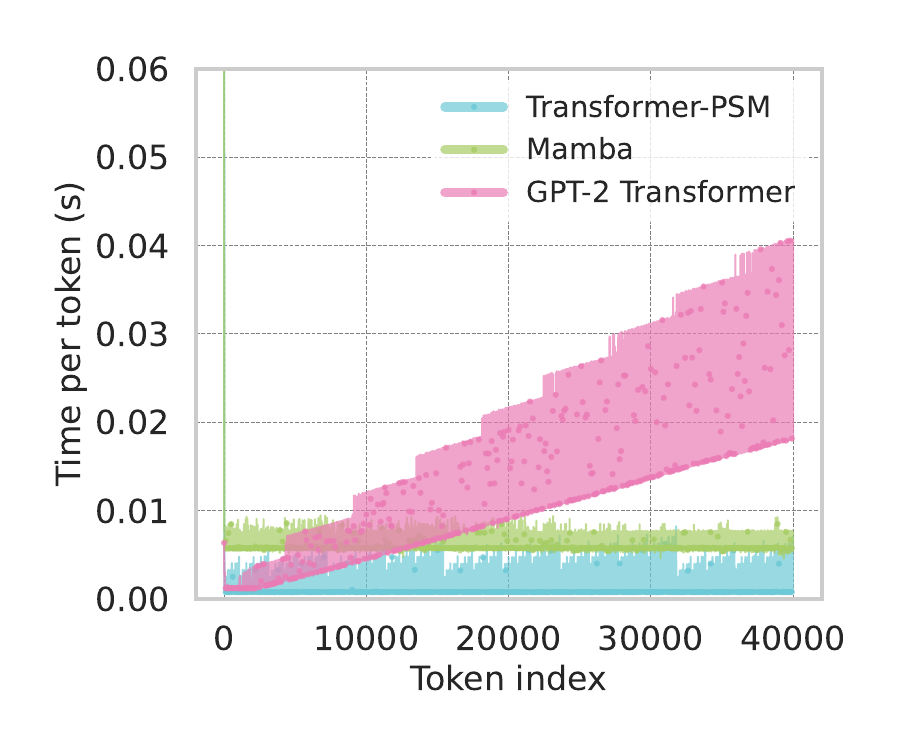}
\vspace{-3mm}
\caption{Inference time per token for Transformer-PSM and GPT-2 Transformer}
\label{fig:wikitext103_inference_speed}
\end{minipage}
\vspace{-2mm}
\end{figure}

Next, we measure per-token latency over 40,000 WikiText-2 tokens for our model versus a 4-layer, 4-head, 256-dimensional GPT-2 baseline. We train Transformer-PSM $(d=768,H=4,L=2)$ $\mathsf{Agg}$, $(d=768,H=4,L=2)$ $\mathsf{Inf}$, thus keeping the parameter count identical to the baseline. As shown in~\Cref{fig:wikitext103_inference_speed}, GPT-2’s inference cost grows linearly with context length ($O(n)$ per token) with KV cache, inflating latency from $\approx 0.002s$ at the start to $\approx 0.04s$ by token 40,000. In contrast, our Transformer-PSM design reuses 64-token chunk summaries, leading to a $O\bigl(2n + \tfrac{n}{32}\log(n/64)\bigr)$ inference cost (as discussed in~\Cref{eq:infer-cost} in~\Cref{app:complexity_psm}), keeping per-token time below $\approx 0.008$s.  For reference, we also include inference time measurement for a Mamba model with $4$ layers, $256$ hidden dimension, with an average inference time per token of $\approx .006$s.

% \subsection{Discussion}
% \label{sec:discussion}
% Our work provides a unique conceptual characterization of existing parallelisable inference-efficient sequence models (and beyond) through the concept of ``prefix scannability''.
% Such a unification seems useful in the current era of sequence modeling, where  many closely related models are developed under different names.
% Moreover, our work deepens the connection between the prefix scan algorithm and efficient sequence models, reinforcing the connection highlighted in prior work \citep{MartinC18}, and building on this view to extend the model design space.

% This algorithmic view of sequence models may provide a framework for, and novel perspectives and insights into designs and development of future sequence models. For example, concurrent work by \cite{guo2025loglinearattention} on ``log linear attention'' is also in line with this view, as they develop a linear attention-style 
% mechanism with log($n$) memory that can be applied to linear attention style mechanisms 
% with structured memory and an efficient chunkwise-parallel primitive.  

%kazuki TODO: more ideas for discussion:
% - Shall we mention that a PSM model essentially "simulates" one layer of modern RNNs?
% we could potentially stack them, i.e., multiple layers of logN shared sub-layers.
% - Another discussion that Morris mentioned on Slack is the larger chunk size case
\vspace{-1mm}
\section{Discussion and Conclusion}
\vspace{-1mm}
\textbf{Discussion.} We give a concise conceptual view of parallelisable, inference-efficient sequence models via \emph{prefix scannability}, unifying many closely related models developed under different names. Our results deepen the link between prefix-scan algorithms and efficient sequence models, extending the design space beyond prior work \citep{MartinC18}.

This algorithmic lens offers a framework for analysing and designing future models. For example, concurrent work on ``log linear attention'' \citep{guo2025loglinearattention} also fits this view, proposing a linear-attention mechanism with $\log n$ memory, structured state, and an efficient chunkwise-parallel primitive.

\textbf{Conclusion.} We formalise \emph{sequential–parallel duality}: models that train in parallel yet decode sequentially. Recent efficient sequence models exhibit this duality and achieve linear-time inference. We characterise them as instances of the classic parallel prefix-scan with a model-specific operator, motivating and analysing the broader class of \emph{parallel scannable models} (PSMs).
% 
% \textbf{Conclusion.} We introduce \textit{sequential-parallel duality}, a fundamental property shared by a broad class of sequence models that support parallel training and sequential inference. All of the recently proposed efficient sequence models belong to this class, with an additional property that their inference time is linear.
% We further propose a formal characterization of these inference-efficient, sequential-parallel models by demonstrating that they are all computable by using the classic parallel prefix scan algorithm with a model-specific associative aggregation operator.
% This motivate us to formally define and analyze
% the class of such ``\textit{parallel scannable models}'' (PSMs).
In particular, we go beyond existing examples of PSMs by defining and empirically studying a novel sequence model based on \textit{non-associative} aggregators. Our experiments suggest that such model may have benefits in length generalization for some tasks, and opens avenues of exploring this design space in light of specific applications.  
Overall, this provides an insightful unification of efficient sequence models, that cannot be found in any prior work.

\section*{Acknowledgements}
This project was partially supported by NSF award CCF-2112665, the MIT Quest for Intelligence, and the Alexander von Humboldt Foundation. S.G. is supported by the MathWorks Engineering Fellowship.

% \sj{if we need space, we could merge discussion and conclusion as one section named exactly that.\\
% references: some have links, some don't. Could just remove all links?}
%kazuki: fixed!

% \subsubsection*{Ethics Statement}
% We have considered the ethical implications of our work and do not foresee any issues.

\newpage
\subsubsection*{Reproducibility Statement}

We provide proofs in~\Cref{app:proofs} and relevant background for all theoretical results in~\Cref{sec:seq_models} and~\Cref{sec:preliminaries}. For experiments, we detail the training protocols in~\Cref{sec:exp}, and algorithm implementations in~\Cref{sec:tpsm}. All datasets are publicly available, and we follow established preprocessing procedures. We release our code with instructions for reproducing the experiments in the paper at https://github.com/muherng/hidden

% \section*{Acknowledgements}
% This project was partially supported by NSF award CCF-2112665, the MIT Quest for Intelligence, and the Alexander von Humboldt Foundation. 

\bibliographystyle{iclr2026_conference}

\newpage
\appendix
\startcontents[appendices]
\printcontents[appendices]{}{1}{\section*{Appendix}}
\newpage

\section{Preliminaries and Depth Conventions}
\label{sec:preliminaries}

\textbf{Depth conventions.}
Throughout this paper we restrict attention to \emph{causal sequence
models whose training and inference graphs can be executed, in the
Random–Access Machine (RAM) model with unbounded fan‑out gates, at
depth (measured by the longest path of synchronous operations)}
%\footnote{Depth is measured by the longest path of synchronous operations.}
\[
  O\!\bigl(\mathrm{polylog}\,n\bigr),
  \quad\text{abbreviated }\tilde O(1).
\]
Whether the hidden polylog factor is $\log n$ or $(\log n)^2$ hinges on the
chosen primitive set—for instance, treating GEMM and \texttt{softmax} as
unit‑time kernels versus expanding them into arithmetic gates.  Classical
fully recurrent networks such as LSTMs \citep{hochreiter1997long} and GRUs \citep{cholearning}, whose forward pass has depth
$\Theta(n)$ and therefore admits \emph{no} sub‑linear parallel schedule, fall
outside this scope.\looseness=-1

Our focus is on \emph{polynomial separations} between the principal model
families: standard Transformers, prefix-scannable models (as we will define in~\Cref{sec:psm}), and linear recurrent
RNNs.  Unless a logarithmic factor is essential to the argument, we suppress
it with the tilde notation.  To characterize our example models, we need to specify how their depth is counted, i.e., 
%To characterize our example models, Next, we need to detail how “depth’’ is counted for
%the Transformer.
%The corresponding question 
in the model of computation: {How tall is their training circuit?}\looseness=-1

\textbf{Transformer.}
      A single self–attention head executes the composite map
      \(
        (\mQ,\mK,\mV)\;\mapsto\;
\operatorname{softmax}\bigl(\mQ \mK^{\top}/\sqrt{d}\bigr)\mV
      \), where $d$ denotes the key/query head dimension. 
      In a Random–Access Machine with unbounded fan‑out gates, the
      \emph{pointwise} linear projections $\vx\mW$ have depth~$1$, but
      the $n\times n$ \emph{matrix multiply} $\mQ\mK^{\!\top}$ and
      every row‑wise \emph{softmax} (vector sum + normalisation)
      require a parallel reduction of $n$ numbers.
      Using a binary tree this costs $\Theta(\log n)$
      depth.\footnote{%
        There are sub‑logarithmic circuits for exact
        matrix multiply~\citep[e.g.][]{valiant1975}, but they are very wide and
        rarely exploited in ML practice; $\log n$ therefore matches realistic
        GPU / TPU kernels.}
      Hence, an $L$–layer Transformer has depth
      \(
          D(n)=\Theta(L\log n).
      \)
      If one \emph{treats the GEMM and softmax kernels as unit‑time
      primitives}, this is often reported as “constant depth,” but strictly
      speaking it is $\tilde O(1)$ (polylogarithmic).\looseness=-1

    \textbf{Mamba, Gated Linear Attention, RWKV.}
      The expensive step is a parallel scan that produces the running
      state.  Its depth is
      $\Theta(\log n)$, and the pointwise gating around it adds
      $O(1)$.  Stacking $L_{\mathsf{agg}}$ such layers gives
      $
D(n)=\Theta\!\bigl(L_{\mathsf{agg}}\log n\bigr).
      $

\section{Additional Proofs}
\label{app:proofs}

\subsection{Modern RNN Layers Fit One Affine Scan}
\label{sec:affine-unification}
% ================================================================

To relate PSMs to recent models, this section shows that a broad family of recent fast‑inference layers
(\autoref{tab:affine-zoo}) are \emph{all} specializations of a single {affine
state‑update} template, i.e. their state kernel can be expressed as an affine bilinear function. This enables \(\textsf{SPD‑}(n,1)\) complexity. 

%kazuki: Define R (and M much earlier)
\begin{definition}[Affine recurrence]\label{def:affine-rec}
Let \((\mathcal M,+,0)\) be an additive group and
\(\blacktriangleright:R\times \mathcal M\to \mathcal M\) a fixed bilinear action of a monoid
\((R,\circ,I)\) on \(\mathcal M\).
A layer is said to have an \emph{affine state update} if its hidden state obeys
\begin{equation}
  \vs_t \;=\; \mE_t \blacktriangleright \vs_{t-1} \;+\; \vf_t,
  \quad
  \vs_{-1}=0,
  \label{eq:affine-main}
\end{equation}
where \((\mE_t,\vf_t)\in R\times \mathcal M\) are (learnable) functions of the current
chunk $x_t$.  That is $\mE_t \coloneqq \mE_{\theta}(\vx_t)$ and $\vf_t \coloneqq \vf_{\theta'}(\vx_t)$ for learnable functions $\mE_\theta$ and $\vf_{\theta'}$.   
\end{definition}

% \subsection{One associative operator}
The models in~\Cref{tab:affine-zoo} all satisfy this affine state update template and all share the following aggregator, which is associative.
\begin{lemma}[Affine aggregator](One associative operator, Affine aggregator) \label{lem:assoc}
Define for \((\mE_i,\vf_i)\in R\times \mathcal M\)
\[
   (\mE_2,\vf_2)\oplus(\mE_1,\vf_1)
  \;=\;
  \bigl(\mE_2\!\circ \mE_1,\;\vf_2 + \mE_2\blacktriangleright \vf_1\bigr),
  \quad
  \ve=(I,0).
\]
Then \((R\times \mathcal M,\oplus,\ve)\) is a monoid—$\oplus$ is associative with
identity $\ve$—and
\[
  (\mE_t,\vf_t)\oplus\cdots\oplus(\mE_0,\vf_0)
  \;=\;
  \bigl(\bar \mE_t,\;\vs_t\bigr),
\]
where \(\vs_t\) is the state given by~\Cref{eq:affine-main} and $\bar \mE_t$ is an auxiliary variable.  
\end{lemma}
\begin{proof}
Straightforward verification using the action axioms; full details in~\Cref{app:proofs}.
\end{proof}
\vspace{-2mm}
% \subsection{Theorem: all affine layers are \texorpdfstring{$\textsf{SPD‑}(n,1,1)$}{SPD‑(n,1,1)}}
Once written in that form, their binary operator is associative, hence each layer is a Prefix–Scannable Model with \(\textsf{SPD‑}(n,1)\) complexity.
%Finally, the following theorem states that all affine layers are \texorpdfstring{$\textsf{SPD‑}(n,1,1)$}{SPD‑(n,1,1)}.

\begin{theorem}\label{thm:affine-spd}
Every layer that satisfies~\Cref{def:affine-rec} is a
Prefix–Scannable Model with chunk size $c=1$, encoder  $\mathsf{Enc}$, aggregator $\mathsf{Agg}$, and inference module $\mathsf{Inf}$ defined as 
\[\vx_i=\mathsf{Enc}(\mC_i) = (\mE_i, \vf_i)\] 
for $i=0,\dots,r-1$
with the aggregator of~\Cref{lem:assoc}.  
\[\mathsf{Agg}(\vx_2, \vx_1) \coloneqq (\mE_2,\vf_2)\oplus(\mE_1,\vf_1)
\]
We define $\mathsf{Inf}(\vs_{i-1}, \mC_i)$ to be the function that takes states and current token and outputs predictions.  Hence these models
admits training work \(\Theta(n)\), parallel depth \(\Theta(\log n)\), and
online inference cost \(O(1)\) time  and \(O(1)\) memory per token:
%\[
  %\boxed{\text{
  the layer is in {\normalfont $\textsf{SPD‑}(n,1)$}.%}
%\]
\end{theorem}
\begin{proof}
Apply the static Blelloch scan to the pairs \((\mE_i,\vf_i)\) using $\oplus$
to obtain every prefix in \(O(n)\) work and \(O(\log n)\) depth.~\Cref{lem:assoc} ensures the scan outputs the correct state \(\vs_t\),
which the inference head may consume chunk‑wise.
During streaming inference, the online left to right scan maintains the same
prefixes with constant work and constant additional memory because $\oplus$ is
associative.\looseness=-1
\end{proof}

% \subsection{Catalogue of affine layers}
\Cref{tab:affine-zoo} shows a catalogue of affine layers. Note that the affine form absorbs normalisation variables common in linear Transformers (e.g.\ running scalars/vectors
\(z_t\); typically running sum of keys \citep{katharopoulos2020transformers} or related variables \citep{beck2024xlstm}) by enlarging the state vector and treating the auxiliary
variable as just another coordinate updated affinely.
The proof of~\Cref{thm:affine-spd} requires no change.

% \textbf{Remark.}  
% The affine form absorbs normalisation variables common in linear Transformers (e.g.\ running scalars/vectors
% \(z_t\); typically running sum of keys \citep{katharopoulos2020transformers} or related variables \citep{beck2024xlstm}) by enlarging the state vector and treating the auxiliary
% variable as just another coordinate updated affinely.
% The proof of Theorem~\ref{thm:affine-spd} requires no change.

% \bigskip

% \textbf{Road map.} So far every example of models relies on an \emph{associative} aggregator.~\Cref{sec:prefix-scan} equips us with the binary‑counter scan, which removes that restriction; the remainder of the paper explores new (non‑associative) aggregators and their empirical behaviour.

\subsection{Examples of Prefix-Scannable Sequence Models}
In the following, we present two families of models whose parallel circuits can be obtained as the computation of a Blelloch parallel scan. 
In fact, it suffices to show that for all family of architectures that are affine in their state, there exists an associative operator $\oplus$ that defines a monoid over which the Blelloch parallel scan operates.  

One type of prefix-scannable models are \textbf{linear time invariant dynamical systems}.

\begin{definition}[LTI Linear Dynamical System] A linear time invariant system is defined by four matrices $(\mA,\mB,\mC,\mD) \in \mathbb{R}^{d \times d}$ defining  
    \begin{align}
        \vs_{t+1} = \mA\vs_t + \mB\vx_t\\
        \vy_t = \mC\vs_t + \mD\vx_t
    \end{align}
Here $\vs_0 = 0$ is the initial state, and $\vs_t$ is the state at time $t \in \mathbb{Z}^+$.  $\vx_t \in \mathbb{R}^d$ is the input at time $t$.   
\end{definition}

\begin{definition}[Associative Operator for Affine State Monoid] For each timestep, let $\vg_t$ be an augmented pair $\vg_t \coloneq (\mE_t,\vf_t) \coloneq (\mA, \mB\vx_t)$ where $\mE_t \in \bbR^{d \times d}$ is a matrix and $\vf_t \in \bbR^d$ is a vector. We define an associative operator $\oplus$ as  
\begin{equation}
(\mE_2,\vf_2) \oplus (\mE_1, \vf_1) = (\mE_2\mE_1, \vf_2 + \mE_2 \vf_1)
\end{equation}
\end{definition}

% \begin{proposition} (Prefix Scannable Sequence Models)
% A sequence model is Prefix Scannable, that is, its outputs can be computed by the Blelloch parallel scan, if (1) the operator $\oplus$ applied to all the $g_i$ over all timesteps computes the state, and (2) $\oplus$ is associative.
% \end{proposition}

To demonstrate that a sequence model is Prefix Scannable,
we must verify two properties.  Firstly, that the operator $\oplus$ applied to all the $g_i$ over all timesteps computes the state.  Secondly that, $\oplus$ is associative. 
    
\begin{lemma}
Let $\mG_t$ be the augmented pair equal to $\oplus$ applied to the sequence of augmented pairs $\vg_1,...,\vg_T$.  Then  
\begin{equation}
\mG_t = (\mE_t,\vf_t) \oplus ... \oplus (\mE_1,\vf_1) = (\mA^{t}, \sum_{k=0}^{t-1} \mA^{t-1-k} \mB\vx_k)
\end{equation}
Secondly for any inputs $\vg_i, \vg_j, \vg_k$ we have $(\vg_i \oplus \vg_j) \oplus \vg_k = \vg_i \oplus (\vg_j \oplus \vg_k)$
\end{lemma}

\begin{proof}
Proof by induction for the equality and straightforward computation for associativity.

We have base case.
\begin{equation}
(\mE_2,\vf_2)\oplus (\mE_1,\vf_1) = (\mA^2, \mB\vx_2 + \mA\mB\vx_1)
\end{equation}
Apply definitions to see this is true for general $t$.  

Proof of associativity.  

$\vg_1, \vg_2, \vg_3$ we have $(\vg_3 \oplus \vg_2) \oplus \vg_1 = \vg_3 \oplus (\vg_2 \oplus \vg_1)$

\begin{align}
\vg_3 \oplus (\vg_2 \oplus \vg_1) = (\mA,\mB\vx_3) \oplus (\mA^2, \mB\vx_2 + \mA\mB\vx_1)  = (\mA^3, \mB\vx_3 + \mA\mB\vx_2 + \mA^2\mB\vx_1)\\
=  (\mA^2, \mB\vx_3 + \mA\mB\vx_2) \oplus (\mA,\mB\vx_1) = (\mA,\mB\vx_3) \oplus (\mA,\mB\vx_2) \oplus (\mA,\mB\vx_1)\\
= (\vg_3 \oplus \vg_2) \oplus \vg_1
\end{align}
\end{proof}

Another type of prefix-scannable models are  \textbf{linear transformers} and their gated variants.

% \subsection{Linear Attention and Gating}
% Gated Linear Attention  

\begin{definition} Gated Linear Attention (GLA) is defined with a states $\vs_1,...,\vs_T \in \bbR^{p \times d}$, inputs $\vx_1,...,\vx_T$, gating function $r: \bbR^d \rightarrow \bbR$, keys $\vk_1,...,\vk_T$, kernel map $\phi: \bbR^{d} \rightarrow \bbR^{p}$  
\begin{align}
\vs_{t} = r(\vx_t) \odot \vs_{t-1} + \phi(\vk_t)\vv_t^T\\
\end{align}
\end{definition}

We observe that GLA has an affine state recurrence.  

\begin{lemma}  Let $\mE_t \in \bbR$ be a scalar that can be computed from $\vx_t$.  Let $\vf_t \in \bbR^{p \times d}$ be a matrix that can be computed from $\vx_t$.  Then the GLA state recurrence is an affine function of the form  
\begin{align}
\vs_{t} = \mE_t \vs_{t-1} + \vf_t\\
\end{align}
In particular, let $\vg_t = (\mE_t,\vf_t)$ be an augmented pair, and let $\oplus$ be an operator defined as 
\begin{align}
(\mE_2,\vf_2) \oplus (\mE_1, \vf_1) = (\mE_2\mE_1, \vf_2 + \mE_2 \vf_1)
\end{align}
Then $\oplus$ is associative, and $\vs_t = \vg_t \oplus ... \oplus \vg_1$.    
\end{lemma}

\begin{proof}
Proof by induction for the equality and straightforward computation for associativity.

First we prove $\vs_t = \vg_t \oplus ... \oplus \vg_1$ by induction.  Consider the base case.  
\begin{align}
\vg_2 \oplus \vg_1 = (r(\vx_2), \phi(\vk_2)\vv_2^T) \oplus (r(\vx_1), \phi(\vk_1)\vv_1^T)\\
= (r(\vx_2) \odot r(\vx_1), \phi(\vk_2)\vv_2^T + r(\vx_2) \odot \phi(\vk_1)\vv_1^T) \\
= (r(\vx_2) \odot r(\vx_1), \vs_2)
\end{align}
Then assuming the identity holds at timestep $t-1$
\begin{align}
\vg_{t} \oplus (r(\vx_{t-1})\odot ... \odot r(\vx_1), \vs_{t-1}) = \\
(r(\vx_{t})\odot ... \odot r(\vx_1), r(\vx_t) \odot \vs_{t-1} + \phi(\vk_t)\vv_t^T)
\end{align}
as desired.  

Then we also check associativity.  

\begin{align}
\vg_3 \oplus (\vg_2 \oplus \vg_1) = \vg_3 \oplus (r(\vx_2) \odot r(\vx_1), \phi(\vk_2)\vv_2^T + r(\vx_2) \odot \phi(\vk_1)\vv_1^T) \\
= (r(\vx_3) \odot r(\vx_2) \odot r(\vx_1), \phi(\vk_3)\vv_3^T + r(\vx_3) \odot \phi(\vk_2)\vv_2^T + r(\vx_2) \odot r(\vx_2) \odot \phi(\vk_1)\vv_1^T)\\
= (r(\vx_3) \odot r(\vx_2), \phi(\vk_3)\vv_3^T + r(\vx_3) \odot \phi(\vk_2)\vv_2^T) \oplus \vg_1 \\
= (\vg_3 \oplus \vg_2) \oplus \vg_1
\end{align}

\end{proof}

% \begin{remark} 
% Note that our claim is that these models admit a parallel circuit corresponding to a Blelleck prefix scan algorithm for a specific choice of operator, which allows us to conclude that these models satisfy SPD.
% However, we do not claim that the prefix scan is used for their training in practice.
% \end{remark}

% \begin{remark}
% A large family of existing models belong to the above cases. TODO detail this more by referring to each model, and their operators.
% \end{remark}

\section{Analytical Details of Chunking in Prefix Scannable Models}
\label{sec:psm_details}
Here we summarize the analytical details of PSMs (\Cref{def:psm}).
%, before moving on to the experiments.
%
The model state of a PSM after $t$ chunks is the Blelloch prefix defined to be 
\(\vs_t=\mathsf{Agg}_{\theta}^{\mathrm{Blelloch}}(\va_0{:}\va_t)\),
and the model outputs
\(
\hat{\vy}_{t}=\mathsf{Inf}_{\phi}\!\bigl(\vs_{t-1},\,\mC_t\bigr).
\)
The same state sequence \(\{\vs_t\}\) can be produced \emph{online} with
$O(\log t)$ memory by replacing the static scan with the
binary‑counter scan of~\Cref{alg:online}.
The corresponding $\mathcal O(\log t)$-parallel depth training algorithm and 
$O(\log t)$-memory online decoding algorithm can be found in \Cref{alg:psm-train} and \Cref{alg:psm-decode}, respectively.
Note that both parallel loops and the \textsc{StaticBlellochScan} have depth
\(\mathcal O(\log t)\), dominated by the static Blelloch scan,
so the whole training pass admits efficient batch execution.

The model has the following properties:

\textbf{Sequential–parallel duality.}
      ~\Cref{alg:psm-train} and~\Cref{alg:psm-decode}
      produce \emph{identical} state sequences~$\{\vs_t\}$
      (\Cref{thm:equivalence}),
      so a PSM trained with the static scan
      can be evaluated online with logarithmic memory.

\textbf{Model family.}
      Choosing $\mathsf{Agg}_{\theta}$ to be associative recovers known scan‑friendly models as a strict subset of PSMs; non‑associative choices (e.g.\ a Transformer block)
      enlarge the design space while retaining online decodability.

\textbf{Complexities.}
      Training: For sequences of length $n$, chunks of size $c$, we have $\mathcal O(n)$ work, $\mathcal O(\log(n/c))$ depth.
      Online inference: $\mathcal O(c)$ amortised work per token and
      $\mathcal O(c\log(n/c))$ memory after $n/c$ chunks.
% \begin{itemize}
% \item \textbf{Sequential–parallel duality.}\;
%       Algorithms~\ref{alg:psm-train} and~\ref{alg:psm-decode}
%       produce \emph{identical} state sequences~$\{s_t\}$
%       (Theorem~\ref{thm:equivalence}),
%       so a PSM trained with the static scan
%       can be evaluated online with logarithmic memory.
% \item \textbf{Model family.}\;
%       Choosing $\mathsf{Agg}_{\theta}$ to be associative recovers known scan‑friendly models as a strict subset of PSMs; non‑associative choices (e.g.\ a Transformer block)
%       enlarge the design space while retaining online decodability.
% \item \textbf{Complexities.}\;
%       Training: $\mathcal O(L)$ work, $\mathcal O(\log(L/c))$ depth.
%       Online inference: $\mathcal O(1)$ amortised work and
%       $\mathcal O(\log t)$ memory after $t$ chunks.
% \end{itemize}

Further details about the computational complexity are detailed below in~\Cref{app:complexity_psm}.

\section{Computational complexity of PSMs}
\label{app:complexity_psm}

Let

\begin{itemize}
\item $n$      – sequence length,
\item $c$      – chunk size ($n=c\!\cdot\!\textit{num\_chunks}$),
\item $L_{\mathsf{agg}}$ – number of Transformer layers in $\mathsf{Agg}_{\theta}$,
\item $L_{\mathsf{inf}}$ – number of Transformer layers in $\mathsf{Inf}_{\phi}$,
\item $d_{\!s}$ and $d_{\!x}$ – hidden widths of the two modules (held constant).
\end{itemize}

Throughout we count one forward–backward pass as a single ``time unit''
and use the usual dense‑attention cost
$\mathcal O(L\,\ell^{2}d)$ for a length‑$\ell$ Transformer block with $L$ layers.
Only the \emph{scaling with $n,c,L_{\mathsf{agg}},L_{\mathsf{inf}}$} is retained;
constant factors in $d_{\!s},d_{\!x}$ are suppressed.

\vspace{4pt}
\noindent\textbf{Training (static Blelloch scan).}
The three parallel loops of~\Cref{alg:psm-train} give

\[
\boxed{\;
T_{\text{train}}
  = \mathcal O\!\bigl(cnL_{\mathsf{agg}} + cn L_{\mathsf{inf}}\bigr)
,\qquad
S_{\text{train}}
  = \mathcal O\!\Bigl(
      cn\,L_{\mathsf{inf}}
      + cn L_{\mathsf{agg}}
    \Bigr)
.  \;}
\tag{C1}\label{eq:train-cost}
\]

\noindent
\emph{Depth} is $\mathcal O\!\bigl(L_{\mathsf{inf}}+\log(n/c)\,L_{\mathsf{agg}}\bigr)$.
because the static Blelloch scan dominates parallel runtime.  Total nonparallel runtime is linear in sequence length $O(cnL_{\mathsf{agg}} + cn L_{\mathsf{inf}})$.  Additional factor of $c$ comes from $c^2$ dense attention for $n/c$ chunks.  

\vspace{4pt}
\noindent\textbf{Online inference (binary‑counter scan).}
Each new chunk incurs

\begin{enumerate}
\item one $\mathsf{Inf}_{\phi}$ call
      $\;\Rightarrow\;$ cost $\mathcal O(L_{\mathsf{inf}}\,c^2)$, and
\item at most $\log(n/c)$ calls to $\mathsf{Agg}_{\theta}$ per chunk
      $\;\Rightarrow\;$ amortised cost $\mathcal O(L_{\mathsf{agg}})$.
\end{enumerate}

Hence, for the whole length‑$n$ stream, we make $n$ calls to $\mathsf{Inf}$ and $\frac{n}{c} $ calls to $\mathsf{Agg}$.  The space at inference is to store the kv-cache for the $c$ tokens in $\mathsf{Inf}$ and the $\log(n/c)$ chunks of $c$ tokens in $\mathsf{Agg}$ 

\[
\boxed{\;
T_{\text{infer}}
  = \mathcal O\!\Bigl(
      nc\,L_{\mathsf{inf}}
      + nc\,L_{\mathsf{agg}}
    \Bigr)
,\qquad
S_{\text{infer}}
  = \mathcal O\!\Bigl(
      c\,L_{\mathsf{inf}}
      + c \log(n/c)\,L_{\mathsf{agg}}
    \Bigr)
.  \;}
\tag{C2}\label{eq:infer-cost}
\]

\noindent
\textbf{Per‑token latency.}  Dividing \eqref{eq:infer-cost} by $n$ gives

\[
\mathcal O\!\Bigl(
      c\,L_{\mathsf{inf}}
      + c\,L_{\mathsf{agg}}
    \Bigr)
\]

work and $\mathcal O(\log n)$ space, confirming constant‑time amortised
decoding under fixed $c$. 

\vspace{4pt}
\noindent\textbf{Remarks.}
\begin{itemize}
\item When $c=\Theta(1)$ (token‑wise chunks) both training and inference are
      linear in $n$ with \emph{constant} memory for $\mathsf{Inf}_{\phi}$ and
      logarithmic memory for $\mathsf{Agg}_{\theta}$.
\item For larger $c$ the quadratic self‑attention of
      $\mathsf{Inf}_{\phi}$ over each chunk dominates work.
\item If $\mathsf{Agg}_{\theta}$ is associative, we may swap
      the static and online scans without affecting costs; thus
      SSMs and gated linear attention inherit \eqref{eq:train-cost}–
      \eqref{eq:infer-cost} as special cases.
\end{itemize}

% \section{Algorithms and Pseudocode}

\section{Beyond Affine State Recurrence, PSM's with General Aggregation: Further Details}
\label{sec:prefix-scan}

The \emph{prefix–scan} (a.k.a.\ parallel prefix) is fundamental to many
parallel algorithms.  When the binary operator is \emph{associative}, the
classic Blelloch scan~\citep{blelloch1990prefix} computes, in
$\mathcal O(n)$ work and $\mathcal O(\log n)$ depth, the same left–to–right
prefix values that a sequential loop would produce.  This section extends the
view to \emph{non–associative} operators such as those expressible by softmax attention. 

But, there is a challenge with non-associativity: the numerical results of straightforward parallel and sequential versions would differ since parenthesisation differs, challenging our duality. %Yet, they remain well defined because the
%Blelloch algorithm fixes a parenthesisation. 
%
\emph{Parenthesisation} here means the explicit placement of parentheses
that fixes \emph{which two elements are combined first} when evaluating a
long chain of binary operations.  For instance,
\[
  a \,\mathsf{Agg}\, b \,\mathsf{Agg}\, c \,\mathsf{Agg}\, d
  \quad\text{may be grouped as}\quad
  ((a \,\mathsf{Agg}\, b)\,\mathsf{Agg}\, c)\,\mathsf{Agg}\, d
  \;\;\text{or}\;\;
  a \,\mathsf{Agg}\, (b \,\mathsf{Agg}\, (c \,\mathsf{Agg}\, d)),
\]
and when \(\mathsf{Agg}\) is \emph{not} associative the two expressions
generally differ.  The Blelloch algorithm removes this ambiguity by
committing to a single, fixed parenthesisation: the full binary tree
generated by its upsweep and downsweep.  All variants we describe—static
and online—evaluate \emph{exactly that same tree}, guaranteeing identical
results even for non‑associative operators.

We first review the {static tree
formulation}, then present an {online variant} that realises \emph{exactly
the same parenthesisation} while using only $\mathcal O(\log n)$ memory.
Throughout, let
\begin{equation}
  \mathsf{Agg}\;:\;\mathcal M \times \mathcal M \;\to\; \mathcal M,
  \qquad
  \text{identity element } \ve\in\mathcal M,
  \tag{A1}\label{eq:agg}
\end{equation}
be an arbitrary binary operator.  {No associativity assumption is required unless stated.}  

First, we introduce the static Blelloch scan which is a ``parallel'' training over sequence elements.  Then we introduce the online binary counter scan which is the ``sequential'' inference over sequence elements that computes prefixes with the same parenthesisation.  The runtime required to run the static Blelloch scan is $T(n) = O(n)$, whereas the amount of space required during the online binary counter scan is $m(n) = O(\log(n))$.  Taken together this analysis gives us PSMs in SPD-$(n, \log(n))$ i.e linear compute during training and nearly linear space during inference.            

\textbf{Static Blelloch Scan (\Cref{alg:static-blelloch}).}
% \SetKw{KwDownTo}{down to}   % algorithm2e keyword (add once in preamble)
~\Cref{alg:static-blelloch} is agnostic to~\Cref{eq:agg}, that is, it is valid for
\emph{any} operator.  When $\mathsf{Agg}$ is not associative, however, the
output for index $t$ no longer equals the sequential recurrence
$\vs_t=\mathsf{Agg}(\va_t,\vs_{t-1})$.  Instead, it is the unique value obtained by
applying $\mathsf{Agg}$ along the fixed binary–tree parenthesisation imposed by
the algorithm.  The next subsection shows how to realise \emph{the same
parenthesisation} online with logarithmic space.

\textbf{Online Binary–Counter Scan (\Cref{alg:online}).}
The online variant processes the stream $a_0,\dots,a_{n-1}$ left to right while
maintaining a \emph{binary counter} of complete mini–trees.  At time
$t$ (0–indexed) the binary expansion of $t{+}1$ reveals which block sizes
$2^k$ are present.  There is at most one mini–tree (its root value) per block
size, hence at most $\lceil\log_2(t{+}1)\rceil$ roots in memory.  Inserting a
new element is identical to adding $1$ to a binary counter: each trailing
\texttt{1} bit triggers a merge with $\mathsf{Agg}$ and a carry to the next
bit.  Aggregating the occupied roots from most- to least-significant bit
(MSB\,$\to$\,LSB) reproduces the value that the static Blelloch tree would hold
after processing the same prefix.

We obtain the following \textbf{correctness and complexity analysis}.
\begin{proposition}[Block invariant]\label{prop:block-invariant}
After processing $t{+}1$ elements ($t\!\ge\!0$), every non–empty
\texttt{root[$k$]} equals the aggregate of the \emph{most–recent}
$2^{k}$ tokens $\vx_{t-2^{k}+1},\dots,\vx_{t}$, and $(t{+}1)$ is divisible by
$2^{k}$.
\end{proposition}
\begin{proof}
By induction on $t$.  The base case $t{=}0$ is immediate.  For the inductive
step, the carry chain merges two adjacent blocks of size~$2^{k}$ precisely
when bit~$k$ flips from \texttt{1} to \texttt{0} in the binary counter.  The
merged value therefore covers the $2^{k+1}$ most recent tokens and is stored at
position $k{+}1$, where divisibility holds.  Untouched positions keep their
invariant.
\end{proof}
% \vspace{-2mm}
%
\EquivBlelloch*
\begin{proof}
By~\Cref{prop:block-invariant} the occupied roots partition the
first $t{+}1$ tokens into blocks whose sizes are decreasing powers of two when
listed MSB\,→\,LSB.  This is exactly the leaf order of the perfect binary tree
used by the static algorithm.  Each block’s internal value was constructed by
the same sequence of merges as in that tree; aggregating the blocks
left–to–right therefore reproduces the tree’s evaluation order and thus its
numeric result.
\end{proof}
% \vspace{-2mm}
\MemoryBound*
\begin{proof}
The binary representation of $t{+}1$ contains at most
$\lfloor\log_2(t{+}1)\rfloor{+}1$ bits, and there is at most one root per bit.
\end{proof}
\textbf{Work.}  Inserting a new element touches exactly the trailing
\texttt{1}--bits of $t$; the expected number of such bits is~$2$, so the
amortised number of $\mathsf{Agg}$ calls per element is constant.

Together,~\Cref{thm:equivalence} and~\Cref{cor:memory} show
that the online binary–counter scan is an \emph{optimal–space, streamable}
realisation of the Blelloch parenthesisation, \textbf{extending prefix–scan techniques
to non–associative operators without increasing asymptotic cost in time}.  This
flexibility enables a larger class of \emph{prefix–scannable models}: sequence models
whose per–token state update is any binary operator that admits
$\mathcal O(\log n)$ space $O(1)$ time online evaluation via the mechanism above. We provide further analytical details of PSMs in \Cref{sec:psm_details}.

\end{document}